\documentclass[10pt,onecolumn]{IEEEtran}
 \usepackage{epsfig,graphics}
 \usepackage{supertabular}
 \usepackage{amsmath}
 \usepackage{amsxtra}
 \usepackage{dblfloatfix}
 \usepackage{amssymb}
 \usepackage{graphicx}
 \usepackage{algorithm}
 \usepackage{algorithmic}
 \usepackage{bbm}
 \usepackage{float}

 \usepackage{flexisym}
 \usepackage{breqn}
 \setkeys{breqn}{breakdepth={1}}
 
 \providecommand{\sqnorm}[1]{{\lVert#1\rVert^{2}_{2}}}
 \providecommand{\norm}[1]{{\lVert#1\rVert_2}}
 \providecommand{\normc}[1]{{\lVert#1\rVert_1}}
 \newcommand{\argmin}{\operatornamewithlimits{arg\ min}}

\begin{document}


\title{Correlation Estimation from Compressed Images\thanks{This work
has been partly supported by the Swiss National Science Foundation,
under grant 200021-118230. Part of this work has been presented in IEEE ICASSP 2011 \cite{Vijay_ICASSP2011}}}

\author{Vijayaraghavan~Thirumalai and Pascal~Frossard \\
        Ecole Polytechnique F\'ed\'erale de Lausanne (EPFL) \\ Signal Processing Laboratory (LTS4), Lausanne, 1015 - Switzerland.
        \\ Email:\{vijayaraghavan.thirumalai, pascal.frossard\}@epfl.ch}

%
%
%

\maketitle

\begin{abstract}
This paper addresses the problem of correlation estimation in sets of compressed images. We consider a framework where images are represented under the form of linear measurements due to low complexity sensing or security requirements. We assume that the images are correlated through the displacement of visual objects due to motion or viewpoint change and the correlation is effectively represented by optical flow or motion field models. The correlation is estimated in the compressed domain by jointly processing the linear measurements. We first show that the correlated images can be efficiently related using a linear operator. Using this linear relationship we then describe the dependencies between images in the compressed domain. We further cast a regularized optimization problem where the correlation is estimated in order to satisfy both data consistency and motion smoothness objectives with a Graph Cut algorithm. We analyze in detail the correlation estimation performance and quantify the penalty due to image compression. Extensive experiments in stereo and video imaging applications show that our novel solution stays competitive with methods that implement complex image reconstruction steps prior to correlation estimation. We finally use the estimated correlation in a novel joint image reconstruction scheme that is based on an optimization problem with sparsity priors on the reconstructed images. Additional experiments show that our correlation estimation algorithm leads to an effective reconstruction of pairs of images in distributed image coding schemes that outperform independent reconstruction algorithms by 2 to 4 dB.
\end{abstract}



\section{Introduction}

\IEEEPARstart{I}n recent years, the increasing popularity of vision sensor networks has led to the generation of huge volume of visual information. This creates the need for effective information processing systems that are able to efficiently compress, analyze and store highly redundant information streams captured by multiple devices. Distributed processing solutions become highly attractive in such a context, as they permit to reduce the communication and computational power requirements in the sensors. The visual information is typically compressed and transmitted independently from each sensor node to a common decoder that jointly processes the correlated information streams. The inter-sensor communication needs are relaxed and the computational burden is shifted to the decoder. The estimation of the image correlation at decoder becomes however crucial in such distributed settings for image reconstruction or analysis tasks.

In this paper, we consider the problem of correlation estimation in a framework where multiple sensors transmit compressed images that have been obtained by a small number of linear projections of the original images, as illustrated in Fig.~\ref{Fig:block_scheme}. Such linear projections typically represent simple measurements in low complexity sensing systems \cite{Donoho,Candes}. We propose a novel solution for correlation estimation at the joint decoder where the analysis is performed directly in the compressed domain in order to avoid expensive image reconstruction tasks. This is especially useful for the analysis applications that do not target image reconstruction. We assume that the correlation between images corresponds to camera or object motion; this can be efficiently represented by optical flow or motion field models. We show that such a correlation model can be described by a linear operator and we further analyze in detail the effect of such operator in the compressed domain. Later, we cast the correlation estimation as a regularized energy minimization problem with constraints on data consistency as well as consistency of the motion field. In particular, we regularize the correlation model such that the motion values in neighboring pixels are similar except at image discontinuities. Such an optimization problem can be solved by Graph Cuts algorithms.
  
We analyze in details the performance of our novel correlation estimation framework. In particular, we study the penalty in the correlation estimation that is due to working in the  compressed domain as opposed to the original image domain as in traditional correlation estimation problems. We show that the penalty decreases when the number of measurements increases and that our algorithm tends to the optimal correlation estimate at high measurement rate. Extensive simulations in distributed stereo and video imaging applications confirm that the proposed solution provides effective estimates of the relative motion between images and even competes with solutions that implement expensive image reconstruction prior to correlation estimation. We finally study the performance of a novel joint reconstruction algorithm that uses our correlation estimates for decoding pairs of images. The joint reconstruction is cast as an optimization problem where the reconstructed images have to satisfy sparsity priors as well as consistency with both the measurements and the correlation estimates. We solve this joint reconstruction problem by effective proximal splitting methods and show that accurate correlation estimation in distributed image representation permits to outperform independent decoding solutions in terms of image quality.


The rest of the paper is organized as follows. Section \ref{sec:related_work} briefly overviews the related work. In Section \ref{sec:framework},  we describe the proposed framework and show how the correlation estimation problem carries out to the compressed domain. Section \ref{sec:prop_ce_algo} describes the proposed correlation estimation algorithm and its performance are analyzed in details on Section \ref{sec:results}. In Section \ref{sec:conc} we draw some concluding remarks. 

\begin{figure*}
\begin{minipage}[t]{1.0\linewidth}
 \centering
 \centerline{\epsfig{figure=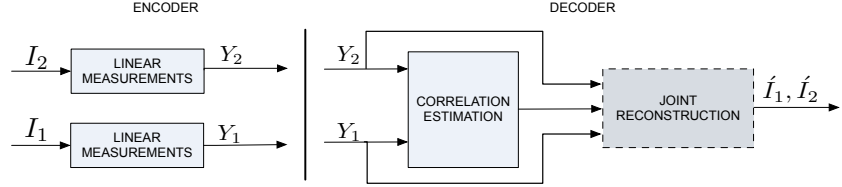,width=11.5cm}}
\end{minipage}
\caption{Schematic representation of the proposed scheme. The images
$I_1$ and $I_2$ are correlated through displacement of scene
objects due to viewpoint change or motion of scene objects. The correlation model is estimated directly in the compressed domain without any intermediate image reconstruction. The correlation information can be used for optional joint reconstruction. }
\label{Fig:block_scheme}
\end{figure*}

\section{Related work} \label{sec:related_work}

This section describes the literature related to the framework proposed in this paper. We first present sensing solutions based on linear measurements. We then review the correlation estimation algorithms in distributed image representation systems and finally we discuss the most relevant works about joint reconstruction of correlated images. 

In recent years, signal acquisition based on random projections received a significant attention in many applications like medical imaging, compressive imaging and even sensor networks. Donoho \cite{Donoho} and Candes \emph{et al.} \cite{Candes} show that the small number of linear measurements contain enough information to reconstruct a sparse or a compressible signal. In particular they show that if a signal has a sparse representation in one basis then it can be recovered from a small number of linear measurements taken on another  (random)  basis that is incoherent with the first one. Essentially, if the signal is $K$-sparse (i.e., if the signal contains $K$ significant components), then one need approximately $cK$ linear measurements (typically $c$ = 3 or 4) to reconstruct the signal with high probability \cite{Candes_imagecoding}. Such results open the door to novel low complexity sensing solutions where the computational complexity for signal reconstruction or analysis is pushed to the decoder. These ideas have been applied to image acquisition \cite{singlepixel_CS, Mun_blockCS, Gan_Eusipco} and later extended to video sequences \cite{Stankovic,Park,Vaswani, Pudlewski}. The effect of measurement quantization and the lossy compression of linear measurements has been studied in \cite{Goyal}.

One of the key characteristics in imaging applications resides in the high correlation between multi-view images or successive images in a video sequence. The correlation could be exploited for effective reconstruction of image sets or for joint analysis tasks in distributed systems. Duarte \emph{et al.} \cite{Duarte_DCS,Duarte_DCS1} have proposed different correlation models for the distributed compression of correlated signals from linear measurements. In particular, they introduce three joint sparsity models (JSM) in order to exploit the inter-signal correlation for the joint reconstruction. They are respectively described by (i) JSM-1, where the signals share a common sparse support plus a sparse innovation part specific to each signal, (ii) JSM-2, where the signals share a common sparse support with different coefficients, and (iii) JSM-3 with a non-sparse common signal with an individual sparse innovation in each signal. These correlation models permit an effective joint reconstruction with a small number of measurements compared to independent reconstruction. These simple joint sparsity models are however not ideal for multi-view images or video sequences, as the correlation model in such scenarios is usually given in the form of disparity or motion vectors respectively. The authors in \cite{Kang_ICASSP,Do_ICIP,YiMa_PCS} have proposed a distributed joint reconstruction scheme for video sequences based on linear measurements. These schemes split the video sequences into key frames  and compressed sensing (CS) frames. The random measurements are computed independently for each compressed sensing frame and are transmitted to the joint decoder. The key frames are intra coded and the joint decoder builds the side information from the intra coded key frames. The generated side information is then used to decode the CS frame by solving an optimization problem which assumes that the prediction error for the CS frame is sparse in an orthonormal basis \cite{Do_ICIP} or block-based adaptive dictionary \cite{YiMa_PCS}. Similar ideas have been used by Trocan \emph{et al.} \cite{Trocan_ICME,Trocan_MMSP} for distributed multi-view compression where all images are however given in the form of linear measurements. The joint decoder first reconstructs all the views by solving a regularized optimization problem. Then, the independently reconstructed views are used to estimate the underlying correlation model in the form of disparity image. The disparity image is used to jointly reconstruct all the views in a multistage refinement framework in which each refinement stage reconstructs a view by assuming that the prediction error is sparse in a dual tree wavelet basis. In our previous work \cite{Vijay_TIP}, we have proposed a methodology to estimate the correlation between pairs of frames where one image serves as a reference image. Unfortunately, reconstructing the reference image in a compressed measurements framework typically requires methods based on solving $l_2$-$l_1$ optimization problems that are highly complex. The works in the literature typically use reference frames that are encoded as intra-frames or reconstructed with complex optimization tools prior to correlation estimation. In this paper, we rather propose to avoid the explicit reconstruction of the images and directly estimate the correlation in the compressed domain. In general, it permits to reduce the computational complexity at decoder, especially in applications where image reconstruction is not necessary.

Finally, other works have recently addressed the problem of joint reconstruction of correlated images given in compressed form. For example, Park \emph{et al.} \cite{Wakin} have proposed an image registration and joint reconstruction algorithm for multi-view images based on manifold lifting. The underlying correlation between views is exploited by constructing an image appearance manifold where the images represent sample points on the manifold; these points are controlled by a few camera parameters (e.g., rotation, translation etc.).  By knowing the initial camera positions the images are jointly reconstructed based on an $l_2$-$l_1$ optimization framework. Then, the reconstructed scene is used to refine the camera parameters and thus both the camera positions and the scene are jointly estimated using alternating minimization techniques. In another framework \cite{Li_jr}, the authors have proposed a joint reconstruction scheme based on a regularized optimization framework. The two regularization terms encourage sparse priors of multi-view images and their difference images. However, the correlation between images is not efficiently exploited as the correlation model in multi-view image settings is usually given in the form of disparity image and not as a sparsity prior of the signal differences. The joint reconstruction scheme proposed in this paper rather builds the correlation model in the form of disparity or motion field and thus facilitates an efficient joint image representation. Furthermore, the correlation model is built directly from the linear measurements in the compressed domain and thus avoids the computational complexity of reconstructing the reference images. The proposed scheme is based on low complexity linear measurements and it provides an interesting flexible solution for distributed processing in vision sensors, targeting applications like object detection, distributed rendering or distributed joint signal reconstruction.  

\section{Distributed Representation of Correlated Images} \label{sec:framework}

\subsection{Framework}

We consider a framework where the images represent a scene at different time instants or from different viewpoints. For the sake of clarity, we consider a pair of images $I_1$ and $I_2$  (with resolution $N$=$N_1 \times N_2$) but the framework extends to larger number of images. These images are represented by linear measurements that correspond to the projection of the image pixel values on a set of coding vectors. Typically, the coding vectors can be constructed from Gaussian or Bernoulli distributions \cite{Donoho} or with a block structure \cite{Gan_Eusipco, Gan_ICASSP} for easier handling and fast sampling of large images. The measurements are transmitted to a joint decoder that can estimate the correlation between the compressed images and possibly perform a joint reconstruction of the image set. The framework is illustrated in Fig.~\ref{Fig:block_scheme}.

In more details, the sensors process images row by row. Let $I_{1,k}$ and $I_{2,k}$ represent the $k^{th}$ row of the images $I_1$ and $I_2$ respectively, and $Y_{1,k}$ and $Y_{2,k}$ represent the linear measurements computed from $I_{1,k}$ and $I_{2,k}$ using the measurement matrices $ \phi_1^k$ and  $ \phi_2^k$ respectively. The measurements $Y_{1,k}$ and $Y_{2,k}$ are computed as
\begin{eqnarray}\label{eqn:meas_row}
Y_{1,k}  = \phi_1^k~I_{1,k}^T, \quad \forall k = 1,2,\ldots,N_1,  \\ \nonumber
Y_{2,k} =  \phi_2^k~I_{2,k}^T, \quad  \forall k = 1,2,\ldots,N_1,
\end{eqnarray}
where $(.)^T$ denotes the transpose operator. It should be noted that $\phi_1^k$ and $\phi_2^k$ are of dimensions $M \times N_2$, where $M << N_2$ is the number of measurements computed for each row in the image. From Eq.~(\ref{eqn:meas_row}) it is easy to check that the measurements $Y_1 = {[Y_{1,1},Y_{1,2}, \ldots Y_{1,{N_1}}]}^T$ and $Y_2 = {[Y_{2,1},Y_{2,2}, \ldots Y_{2,{N_1}}]}^T$ can be computed as
\begin{eqnarray}\label{eqn:measurement}
\left[ \begin{array}{c}
                 Y_{i,1}  \\
                 Y_{i,2}   \\
                \vdots \\
                 Y_{i,{N_1}}  \end{array} \right]  
                  = \Phi_i
                  \underbrace{\left[ \begin{array}{c}
                 I_{i,1} ^T \\
                 I_{i,2} ^T  \\
                \vdots\\
                 I_{i,{N_1}}^T  \end{array} \right] }_{I_i}, \quad \forall i \in \{1,2\},
                 \end{eqnarray}
where $\Phi_i$ is the measurement matrix used to sample the $i^{th}$ image $\forall i = \{1,2\}$. It is represented as 
\begin{eqnarray}
\Phi_i  = {\left[ \begin{array}{cccc}
               \phi_i^1 & 0 & \ldots &0 \\
                0 &  \phi_i^2 & \ldots &0  \\
               \vdots &\vdots &\ddots &\vdots \\
                 0 & 0 &\ldots & \phi_i^{N_1}  \end{array} \right]}_{K \times N}, \forall i \in \{1,2\},
\end{eqnarray}
where $K = MN_1$, $N = N_1N_2$ and $K/N$ represents the measurement rate.

\subsection{Correlation Model} 
\label{sec:relation_images}

In the above settings, the correlation between images can be mainly explained by the relative displacement of objects in the scene. This can be modeled effectively by the optical flow that determines the amount of displacement of objects or pixels in different images. We show now, how such a correlation model can be described by a linear operator. Let $\bold{m}^h$ and $\bold{m}^v$  represent the horizontal and vertical motion components. As the visual objects in the images $I_1$ and $I_2$ are displaced, it is likely that the pixel at position $\bold{z}= (k,l)$ in one image moves to $\bold{z^\prime}=(k+\bold{m}^h(k,l),l+\bold{m}^v(k,l))$ in the second image. Thus, the images $I_1$ and $I_2$ can be simply related by a linear operator $\mathcal{T}$ that changes the coordinate system from $(k,l)$ in the first image to $(k+\bold{m}^h(k,l),l+\bold{m}^v(k,l))$ in the second image, i.e., 
\begin{eqnarray}\nonumber
I_2 &= &\mathcal{T}\{I_1\}  \\  \label{eqn:corr_linear}
I_{2,k}(l) & = & I_{1,(k+\bold{m}^h(k,l)}(l+\bold{m}^v(k,l)).
\end{eqnarray}
For mathematical convenience we use an equivalent representation of Eq.~(\ref{eqn:corr_linear}) in the form of matrix multiplication:
\begin{equation} \label{eqn:corr_matrix}
I_{2,k}^T= A^k~ \underbrace{\left[ \begin{array}{c}
                 I_{1,1} ^T \\
                 I_{1,2} ^T  \\
                \vdots \\
                 I_{1,{N_1}}^T  \end{array} \right]}_{I_1} , \quad \forall k = 1,2,\ldots,N_1,
\end{equation}
where  $A^k$ is a matrix of dimensions $N_2 \times N_1N_2$ whose entries are determined by the horizontal and vertical components of the motion field in the $k^{th}$ row of pixels, i.e., $\bold{m}^h(k,.)$ and $\bold{m}^v(k,.)$. The elements of the matrix $A^k$ are given by 
\begin{equation}\label{eqn:matA}
A^k(l,l+\beta_1+\beta_2N_2) =    \left\{
                              \begin{array}{ll}
                             1 & \mbox{if \ } \bold{m}^h(k,l) = \beta_1,\bold{m}^v(k,l) = \beta_2 \\
                             0 &  \mbox{otherwise} \end{array} \right.
\end{equation}
If $l+\beta_1+\beta_2N_2 > N_1N_2$ (e.g., at image boundaries), we set $l+\beta_1+\beta_2N_2 = N_1N_2$ so that the dimensions of the matrix $A^k$ stays $N_2 \times N_1N_2$. It is easy to check that the matrix $A^k$ formed using Eq.~(\ref{eqn:matA}) contains only one '$1$' in each row; this implies $I_{2,k}(l)= I_{1,k+\beta_1}(l+\beta_2)$ if $A^k(l,l+\beta_1+\beta_2N_2) =1$. The action of the matrix $A^k$ shifts the pixels in $I_{1}$ according to the motion given as $\bold{m}^h(k,.)$ and $\bold{m}^v(k,.)$ and forms an estimate of the image $I_{2,k}$. It should be noted that the matrix $A^k$ is completely determined by the $k^{th}$ row of motion vectors $\bold{m}^h(k,.)$ and $\bold{m}^v(k,.)$.

The relation given in Eq. (\ref{eqn:corr_matrix}) can be extended to all rows of the image $I_2$. The images $I_1$ and $I_2$ are finally related by a linear operator $A$ such that $I_2 = A \ I_1$ which can be written as
\begin{eqnarray}\label{eqn:I2AI1}
\underbrace{\left[ \begin{array}{c}
                 I_{2,1}^T  \\
                 I_{2,2}^T  \\
                \vdots \\
                 I_{2,{N_1}}^T  \end{array} \right] }_{I_2}
                  =  \underbrace{\left[ \begin{array}{c}
                 A^1 \\
                 A^2  \\
                \vdots \\
                 A^{N_1} \end{array} \right] }_{A}
                 \underbrace{ \left[ \begin{array}{c}
                 I_{1,1} ^T \\
                 I_{1,2} ^T  \\
                \vdots \\
                 I_{1,{N_1}}^T  \end{array} \right]  }_{I_1}.
                 \end{eqnarray}
This relation is illustrated on the lefthand side of Fig.~\ref{Fig:relationbetA_B}.

\begin{figure}[t!]
 \centering
 \centerline{\epsfig{figure=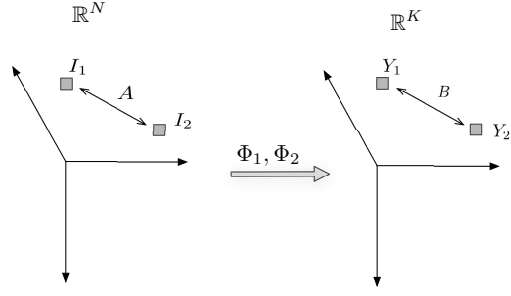,width=7cm}}
\caption{Illustration of the relation between the matrices A and B. On the left the images $I_1, I_2 \in \mathbb{R}^N$ are related using the matrix $A$ where $N = N_1 N_2$. In the compressed domain, the measurement vectors  $Y_1, Y_2 \in \mathbb{R}^K$ are related using the matrix $B$ where $K = MN_1$. The matrices  $A$ and $B$ can be related by $B \approx  \Phi_2 A \Phi_1^{\dagger}$ where $\Phi_i$'s are the sensing matrices.}
\label{Fig:relationbetA_B}
\end{figure}

\subsection{Correlation in the measurement domain} 
\label{sec:Rel_meas}

We now extend the above correlation model in the compressed domain. Without loss of generality, we first assume that the measurements $Y_{1}$ and $Y_{2}$ can be related by a linear transformation $B$, i.e., 
\begin{eqnarray} \label{eqn:Y2BY1}
\underbrace{\left[ \begin{array}{c}
                 Y_{2,1}  \\
                 Y_{2,2}  \\
               \vdots \\
                 Y_{2,{N_1}}  \end{array} \right] }_{Y_1}
                  =  \underbrace{\left[ \begin{array}{c}
                 B^1 \\
                 B^2  \\
                \vdots \\
                 B^{N_1} \end{array} \right] }_{B}
               \underbrace{  \left[ \begin{array}{c}
                 Y_{1,1}  \\
                 Y_{1,2}   \\
                \vdots \\
                 Y_{1,{N_1}}  \end{array} \right]}_{Y_2},
                 \end{eqnarray}
where $B^k$, $\forall k = 1,2,\ldots,N_1$ is a matrix with dimensions $M \times MN_1$, i.e., the measurements $Y_{2,k}$ can be related to $Y_1$ as
\begin{equation} \label{eqn:meas_matrix_row}
Y_{2,k}= B^k~Y_{1}, \;  \forall k = 1, 2,\ldots,N_1.
\end{equation}
Any two vectors $Y_{1}, Y_{2} \in \mathbb{R}^{MN_1}$ can be related by a linear transformation $B$ as long as $Y_{1} \neq \bold{0}$, which is the case in our framework. However, this linear transformation $B$ a priori does not have any special form. We are interested in understanding the relation between this matrix and the matrix $A$ that shifts the pixels between images $I_1$ and $I_2$. Pre-multiplying Eq.~(\ref{eqn:I2AI1}) by $\Phi_2$ on both sides, one can write 
\begin{equation} \label{eqn:relation1}
 Y_{2} = \Phi_2 I_{2} = \Phi_2 AI_{1}.  
 \end{equation}
In addition, by replacing $Y_1 =\Phi_1I_1 $ in Eq.~(\ref{eqn:Y2BY1}), one obtains 
 \begin{equation} \label{eqn:relation2}
 Y_{2}= BY_{1} = B \Phi_1 I_{1}.
\end{equation}
From Eqs.~(\ref{eqn:relation1}) and~(\ref{eqn:relation2}) the relation between $B$ and $A$ can finally be given as
\begin{equation}\label{eqn:relation_A_B}
 B \Phi_1 = \Phi_2 A. 
\end{equation}
This forms an over-determined system of linear equations, as the number of unknown in matrix $B$ is smaller than the number of equations in $\Phi_2A$. In this case, the optimal matrix $\hat{B}$ that minimizes $\norm{B \Phi_1 - \Phi_2 A}$ is given by
\begin{equation}\label{eqn:best_B}
 \hat{B} = \Phi_2 A {\Phi_{1}}^\dagger,
\end{equation}
where $\dagger$ denotes the pseudo-inverse operator. As the rows in $\Phi_1$ are generally orthonormal, the pseudo-inverse can be computed using the transpose operator, i.e., ${\Phi_{1}}^\dagger = {\Phi_{1}}^T$. Substituting $\hat{B} = \Phi_2 A {\Phi_{1}}^T$ in Eq.~(\ref{eqn:Y2BY1}) the relation between the measurements becomes 
\begin{eqnarray}\label{eqn:meas_matrix_row1}
Y_{2} &\approx  &\Phi_2 A {\Phi_{1}}^\dagger Y_{1} = \Phi_2 A {\Phi_{1}}^T Y_{1}  \\ \label{eqn:temp1}
Y_{2,k} &\approx  &\phi_2^k A^k {\Phi_{1}}^T Y_{1},  \; \forall k  =1,2,\ldots,N_1.
\end{eqnarray}
Eq.~(\ref{eqn:temp1}) comes from the fact that measurements are computed across rows of pixels in our framework. The relationship between the matrices  $A$ and $B$ is illustrated in Fig.~\ref{Fig:relationbetA_B}, where the matrices  $A$ and $B$ are used to relate the points in the original and compressed domains, respectively. In the next section, we propose an algorithm for estimating the correlation model directly from the linear measurements $Y_1$ and $Y_2$. 


\section{Correlation estimation from linear measurements} \label{sec:prop_ce_algo}

\subsection{Regularized energy minimization problem}

We propose in this section a method for estimating the correlation between images from the compressed measurements without any explicit image reconstruction step. The objective is to compute a flow or motion field that represents the motion between images $I_1$ and $I_2$. We denote this flow field as $\mathcal{M} = (\bold{m}^h,\bold{m}^v)$, where $\bold{m}^h$ and $\bold{m}^v$ are horizontal and vertical components of the motion, respectively. The problem then consists in finding the value of the flow field $\mathcal{M}$ at each pixel position $\bold{z} = (k,l)$ such that the estimated correlation is consistent with the measurement vectors $Y_1$ and $Y_2$. At the same time, the motion field has to be piecewise smooth in order to model consistent motion of visual objects. We propose to cast the correlation estimation as a regularized energy minimization problem where the energy $E(\mathcal{M})$ is composed of a data term $E_d(\mathcal{M})$ and a smoothness term $E_s(\mathcal{M})$. The optimal mapping  $\mathcal{M}^*$ is obtained by minimizing the energy function $E(\mathcal{M})$ as 
\begin{equation} \label{eqn:energy_min}
\mathcal{M}^* = \argmin_{\mathcal{M}} ~E(\mathcal{M}) = \argmin_{\mathcal{M}} [E_d(\mathcal{M}) + \lambda E_s(\mathcal{M})],
\end{equation}
where $\lambda$ balances the importance of the data and smoothness terms. 

We now discuss in more details the components of the energy function. The smoothness term measures the penalty of assigning different motion values to the adjacent pixels. We write it as
\begin{equation} \label{eqn:smoothcost}
E_s (\mathcal{M}) = \sum_{\bold{z},\bold{z^\prime} \in \mathcal{N}} V_{\mathcal{M}}(\bold{z},\bold{z^\prime}),
\end{equation}
where $\bold{z},\bold{z^\prime}$ are neighbour pixels in the 4-pixel neighbourhood denoted by $\mathcal{N}$. The term $V_{\mathcal{M}}(\bold{z},\bold{z^\prime})$ is given as
\begin{equation}
V_{\mathcal{M}}(\bold{z},\bold{z^\prime}) = min(|{\bold{m}^h}(\bold{z})-{\bold{m}^h}(\bold{z^{\prime}})| + |{\bold{m}^v}(\bold{z})-{\bold{m}^v}(\bold{z^{\prime}})|, \tau),
\end{equation}
where $\tau$ sets an upper level on the smoothness penalty that helps preserving the discontinuities \cite{Veksler}. 

Next, the data function measures the consistency of a particular motion value for pixel $\bold{z}$ with the vectors $Y_1$ and $Y_2$. Classically, the accuracy of the motion values is evaluated with the original images \cite{Baker_flow} and the data cost is typically given as
\begin{equation} \label{eqn:datacost}
\tilde{E}_d(\mathcal{M})  =  \sum_{k=1}^{N_1} \sum_{l=1}^{N_2} \Delta_{\mathcal{M}}(k,l),
\end{equation} 
where $ \Delta_{\mathcal{M}} (k,l) = \sqnorm{ I_{2,k}(l) - I_{1,k+\bold{m}^h(k,l)}(l+\bold{m}^v(k,l))}$ represents the error of matching the pixel at position $(k,l)$ in the second image with a pixel in the first image that is selected according to the motion information. As discussed in Section \ref{sec:relation_images}, the effect of motion between images can be captured by a linear operator $A$ that is a composition of sub-matrices $A^k$. We can therefore rewrite Eq.~(\ref{eqn:datacost}) as 
\begin{equation} \label{eqn:datacost_images}
\tilde{E}_d(\mathcal{M}) =  \sum_{k=1}^{N_1} \sqnorm{I_{2,k}^T - A^k_{\mathcal{M}} I_{1}},
\end{equation}
where the sub-matrix $A^k$ depends on the motion field $\mathcal{M}$ according to Eq.~(\ref{eqn:matA}). In the rest of the development, we drop the index $\mathcal{M}$ as the dependency on the motion field is clear from the context. 

In our framework however we do not have access to the original images, but only to the measurement vectors $Y_1$ and $Y_2$. We thus approximate the data cost $\tilde{E}_d(\mathcal{M})$ by $E_d(\mathcal{M})$ that is computed directly from the measurement vectors. It can be written as
\begin{eqnarray}  \label{eqn:datacost_der_a}
\tilde{E}_d(\mathcal{M})  &=  & \sum_{k=1}^{N_1} \sqnorm{I_{2,k}^T - A^k I_{1}} \\   \label{eqn:datacost_der_b}  
& \approx & \sum_{k=1}^{N_1} \sqnorm{Y_{2,k} - B^k Y_{1} }      \\  \label{eqn:datacost_der_c}
        &\approx & \sum_{k=1}^{N_1} \sqnorm{Y_{2,k} -\phi_2^k A^k {\Phi_1}^T  Y_{1}} \\ \label{eqn:datacost_measu}     
         & = & E_d(\mathcal{M}).
\end{eqnarray}
Note that the data cost approximation due to working in the compressed domain comes from the relation between the matrices  $A$ and $B$ that is given as $B \approx \Phi_2A\Phi_1^T$. We study in details the effect of this approximation in the next section. 

We can finally rewrite the regularized energy objective function for the correlation estimation problem. It reads as
\begin{eqnarray}\label{eqn:final_energy}
E(\mathcal{M}) =  \sum_{k=1}^{N_1} \sqnorm{Y_{2,k} -\phi_2^k A_{\mathcal{M}}^k {\Phi_1}^T  Y_{1}} +\lambda \sum_{\bold{z},\bold{z^\prime} \in \mathcal{N}} V_{\mathcal{M}}(\bold{z},\bold{z^\prime}).
\end{eqnarray} 
This cost function is used in the optimization problem of Eq. (\ref{eqn:energy_min}), which becomes a non-convex problem. The search space is discrete and usually constrained by limits on each of the motion values which typically define a motion search window. The solution to this problem can be determined with strong optimization techniques based on Graph Cuts \cite{Veksler, Boykov_GC} or Belief propagation \cite{Belief_prop}. A comprehensive overview of various energy minimization techniques is summarized in \cite{Szeliski_MRF}. In this paper, we use an optimization algorithm based on $\alpha$-expansion mode in Graph Cuts whose complexity is bounded by a low order polynomial \cite{Boykov_GC}. 

Finally, it should be noted that the correlation estimation can also be performed by block of pixels. In this case, each block of pixels is assumed to move in a coherent way, and the objective of the correlation estimation problem is to compute one motion vector per block. The data cost function can then be modified in a straightforward way by imposing the same motion vector for all the pixels in a block. The
smoothness function is also modified in this case such that it penalizes the difference between the motion values of adjacent blocks rather than neighboring pixels. The optimization problem keeps the same form in block-based motion estimation but the search space is dramatically reduced as the number of motion vectors is smaller.

\subsection{Compressed domain penalty} 
\label{sec:optimal_property}

We now discuss the penalty of estimating the correlation from measurements instead of original images. When the correlation $\mathcal{M}$ is given, this penalty corresponds to the difference between the values of the regularized energy function of Eq. (\ref{eqn:energy_min}) that is evaluated from original images or respectively measurements. Recall that the smoothness cost function $E_s (\mathcal{M})$ depends only on the correlation (see Eq. (\ref{eqn:smoothcost})). Therefore, the estimation penalty is identical to the error between the data cost functions $\tilde{E_d}(\mathcal{M})$ and $E_d(\mathcal{M})$ that are computed in the original and compressed domains respectively. We first show that the penalty is bounded. Then, we show that the penalty decreases monotonically when the number of measurements increases. 

\newtheorem{prop}{Proposition} 

\begin{prop}\label{prop:one}
The penalty of estimating the correlation from measurements is bounded. In particular we have $|{(1-\delta)^2}\tilde{E_d}(\mathcal{M})- C_l| \leq E_d(\mathcal{M}) \leq (1+\delta)^2 \tilde{E_d}(\mathcal{M}) + C_u$, where $\delta>0$, $C_l = \eta^2 + 2 (1-\delta) \alpha \eta$, $C_u = \eta^2 + 2 (1+\delta) \alpha \eta$, $\alpha=  \sum_{k = 1}^{N_1}\|I_{2,k}^T - A^k I_1 \|_2 $, ${\eta = \sum_{k=1}^{N_1} \sigma_{max}(A^k) \sum_{p = k-w_y}^{k+w_y}\norm{\tilde{I}_{1,p} -  I_{1,p}}}$ with $\tilde{I}_1 = \Phi_1^T Y_1$.
\end{prop}

\begin{proof}
Let us assume that $\mathcal{M}$ or equivalently $A$ is given. Then, the points $\mathcal{P} = \{I_{2,k}^T, A^k I_1 : k = 1,2,\cdots,N_1\}$ forms a finite set. According to the Johnson-Lindenstrauss (JL) lemma, the distances between points in $\mathcal{P}$ are preserved in the measurement domain $\mathbb{R}^M$ when $M = \mathcal{O}( \delta^{-2}\;log |\mathcal{P}|)$ measurements are computed with a measurement matrix $\phi_2^k$ \cite{Baraniuk_JLlemma,Gan_JLlemma}, where $|\mathcal{P}|$ denotes the number of points in $\mathcal{P}$. Mathematically, the JL-embedding is given as 
\begin{equation} \label{eqn:JLlemma}
(1-\delta)\norm{I_{2,k}^T - A^k I_1}  \leq \norm{\phi_2^k I_{2,k}^T -  \phi_2^k A^k I_1}  \leq (1+\delta) \norm{I_{2,k}^T - A^k I_1},
\end{equation}
for a positive constant $\delta$.  It should be noted that, when the measurement matrix $\phi_2^k$ satisfies Eq.~(\ref{eqn:JLlemma}), then with high probability it satisfies the restricted isometry property (RIP). For more details related to the connection between the JL-lemma and the RIP we refer the reader to \cite{Baraniuk_JLlemma,Gan_JLlemma}.  Eq.~(\ref{eqn:JLlemma}) holds with high probability not only for measurement matrices constructed using Gaussian and Bernoulli distributions but also for structured measurement matrices constructed using orthonormal bases, e.g., DCT, FFT \cite{Gan_JLlemma}. In our experiments we construct measurement matrices using structured FFT. 

For a given row index $k$, the term  $Y_{2,k} -  \phi_2^k A^k\Phi_1^{T} Y_1$  in Eq.~(\ref{eqn:datacost_der_c}) we can write as  
\begin{eqnarray} \nonumber 
Y_{2,k} -  \phi_2^k A^k\Phi_1^{T} Y_1  &=&\phi_2^k I_{2,k}^T -  \phi_2^k A^k\Phi_1^T \Phi_1I_1 \\ \nonumber
&=&\phi_2^k I_{2,k}^T - \phi_2^k A^kI_1 + \phi_2^k A^kI_1 -\phi_2^k A^k\Phi_1^T  \Phi_1 I_1 \\ \label{eqn:thm_step1}
&=&\phi_2^k I_{2,k}^T - \phi_2^k A^kI_1 + E I_1,
\end{eqnarray}
where $E = \phi_2^k A^k -\phi_2^k A^k\Phi_1^T  \Phi_1 $. The term $\norm{Y_{2,k} -  \phi_2^k A^k\Phi_1^{T} Y_1}$ can be upper bounded as 
\begin{eqnarray} \nonumber
\norm{Y_{2,k} -  \phi_2^k A^k\Phi_1^{T} Y_1} &=& \norm{\phi_2^k I_{2,k}^T - \phi_2^k A^kI_1 + E I_1} \\\nonumber
&\leq& \norm{\phi_2^k I_{2,k}^T - \phi_2^k A^k I_1} + \norm{E I_1}\\  \label{eqn:ubound}
&\leq& (1+\delta) \norm{I_{2,k}^T -  A^kI_1} + \norm{E I_1},
\end{eqnarray}
where the last inequality is derived from Eq. ~(\ref{eqn:JLlemma}). Similarly the term $\norm{Y_{2,k} -  \phi_2^k A^k\Phi_1^{T} Y_1}$ can be lower bounded as 
\begin{eqnarray} \nonumber 
\norm{Y_{2,k} -  \phi_2^k A^k\Phi_1^{T} Y_1} &= &\norm{\phi_2^k I_{2,k}^T - \phi_2^k A^kI_1 + E I_1} \\  \nonumber 
&=& \norm{\phi_2^k I_{2,k}^T - \phi_2^k A^kI_1 - (-E I_1)} \\ \label{eqn:triainequality}
&\geq& |\norm{\phi_2^k I_{2,k}^T - \phi_2^k A^kI_1}  - \norm{E I_1}| \\  \label{eqn:lbound}
&\geq& |(1-\delta) \norm{I_{2,k}^T -  A^kI_1} - \norm{E I_1} |,
\end{eqnarray}
where Eq. (\ref{eqn:triainequality}) follows from $\norm{x-y} \geq |\norm{x}-\norm{y}|$, and Eq. (\ref{eqn:lbound}) is derived from Eq. ~(\ref{eqn:JLlemma}). The term $\norm{EI_1}$ in Eqs.~(\ref{eqn:ubound}) and (\ref{eqn:lbound})  can also be bounded as 
\begin{eqnarray}  \nonumber 
\norm{EI_1}&= &\norm{\phi_2^k A^k I_1 -\phi_2^k A^k\Phi_1^T \Phi_1 I_1}\\  \nonumber 
& = &\norm{\phi_2^k A^k(\Phi_1^T \Phi_1 I_1 -  I_1)}  \\  \label{eqn:phicompact}
& \leq& \norm{A^k(\tilde{I}_1 -  I_1) } \\ \nonumber
& \leq & \norm{A^k} \sum_{p = k-w_y}^{k+w_y}\norm{\tilde{I}_{1,p} -  I_{1,p}}  \\  \label{eqn:E_bound}
&=&  \sigma_{max}(A^k) \sum_{p = k-w_y}^{k+w_y}\norm{\tilde{I}_{1,p} -  I_{1,p}}  =  \eta_k,
\end{eqnarray}
where Eq.~(\ref{eqn:phicompact}) follows from $\norm{\Phi x} \leq \norm{x}$, as $\Phi$ is a non expanding operator \cite{Baraniuk_JLlemma} and $\tilde{I}_{1} = \Phi_1^T \Phi_1 I_1$ is the pre-image of $I_1$.  $\sigma_{max}(A^k)$ in Eq.~(\ref{eqn:E_bound}) denotes the largest singular value of $A^k$. The summation in Eq.~(\ref{eqn:E_bound}) is carried out from rows $k-w_y$ to $k+w_y$ as the search window is usually bounded, where $w_y$ is the admissible search size along the vertical direction. Combining Eq.~(\ref{eqn:ubound}), Eq.~(\ref{eqn:lbound}) and Eq.~(\ref{eqn:E_bound}), and by taking squares we get  for each row of pixels
\begin{equation} \label{eqn:boundperrow}
\left( |(1-\delta)\norm{I_{2,k}^T -  A^k I_1} - \eta_k| \right )^2  \leq \sqnorm{Y_{2,k} -  \phi_2^k A^k\Phi_1^T Y_1} \leq \left ( (1+\delta) \norm{I_{2,k}^T -  A^k I_1} + \eta_k\right )^2.
\end{equation}
Adding the second and third inequality terms of  Eq.~(\ref{eqn:boundperrow}) for all rows $k = 1, 2,\cdots,N_1$ results in 
 \begin{dgroup*} \nonumber
 \begin{dmath*}
E_d(\mathcal{M}) = \sum_{k=1}^{N_1}\sqnorm{Y_{2,k} -  \phi_2^k A^k \Phi_1^T Y} 
\end{dmath*}
 \begin{dmath*}
 \leq (1+\delta)^2  \sum_{k=1}^{N_1}\sqnorm{ I_{2,k}^T -  A^k I_1} +\sum_{k=1}^{N_1} \eta_k^2 + \sum_{k=1}^{N_1} 2(1+\delta) \eta_k \norm{ I_{2,k}^T -  A^k I_1} 
\end{dmath*}
 \begin{dmath*}
\leq (1+\delta)^2 \tilde{E}_d(\mathcal{M})  + \left ( \sum_{k=1}^{N_1} \eta_k \right)^2 + 2(1+\delta) \underbrace{\sum_{k=1}^{N_1} \eta_k}_{\eta} \underbrace{\sum_{k=1}^{N_1} \norm{ I_{2,k}^T -  A^k I_1}}_{\alpha}
\end{dmath*}
 \begin{dmath*}
= (1+\delta)^2 \tilde{E}_d(\mathcal{M})  + \eta^2 + 2(1+\delta) \eta \alpha
 \end{dmath*}
 \begin{dmath} \label{eqn:dataallpixel_ubound}
=  (1+\delta)^2 \tilde{E}_d(\mathcal{M})  + C_u,
\end{dmath}
\end{dgroup*}
where
\begin{equation} \label{eqn:term_Cu} 
C_u = \eta^2 + 2(1+\delta) \eta \alpha,
\end{equation}
and 
\begin{equation} \label{eqn:term_eta} 
\eta = \sum_{k=1}^{N_1} \eta_k = \sum_{k=1}^{N_1} \sigma_{max}(A^k) \sum_{p = k-w_y}^{k+w_y}\norm{\tilde{I}_{1,p} -  I_{1,p}}.
\end{equation} 
In a similar way, from the first and second inequality terms of Eq.~(\ref{eqn:boundperrow}) we get
 \begin{dgroup*} \nonumber
 \begin{dmath*}
E_d(\mathcal{M}) = \sum_{k=1}^{N_1}\sqnorm{Y_{2,k} -  \phi_2^k A^k \Phi_1^T Y} 
\end{dmath*}
 \begin{dmath*}
 \geq  \sum_{k=1}^{N_1} \left\{ \left | (1-\delta)^2\sqnorm{ I_{2,k}^T -  A^k I_1} + \eta_k^2 - 2 (1-\delta) \eta_k \norm{ I_{2,k}^T -  A^k I_1} \right | \right\} 
\end{dmath*}
 \begin{dmath*}
\geq \left| \sum_{k=1}^{N_1} (1-\delta)^2\sqnorm{ I_{2,k}^T -  A^k I_1} + \sum_{k=1}^{N_1} \eta_k^2 - \sum_{k=1}^{N_1} 2(1-\delta) \eta_k\norm{ I_{2,k}^T -  A^k I_1} \right| 
\end{dmath*}
 \begin{dmath*}
 \geq | (1-\delta)^2  \tilde{E}_d(\mathcal{M}) - \left( \sum_{k=1}^{N_1} \eta_k \right )^2 -  2(1-\delta) \underbrace{\sum_{k=1}^{N_1}\eta_k}_{\eta} \underbrace{\sum_{k=1}^{N_1}\norm{ I_{2,k}^T -  A^k I_1}}_{\alpha} | 
 \end{dmath*}
 \begin{dmath*}
=  |(1-\delta)^2 \tilde{E}_d(\mathcal{M})  - (\eta^2 +2(1-\delta) \eta \alpha) | 
\end{dmath*}
 \begin{dmath} \label{eqn:dataallpixel_lbound} 
=  |(1-\delta)^2 \tilde{E}_d(\mathcal{M})  - C_l|,
\end{dmath}
\end{dgroup*}
where
\begin{equation} \label{eqn:term_Cl} 
C_l = \eta^2 + 2(1-\delta) \eta \alpha,
\end{equation}
and $\eta$ is given in Eq.~(\ref{eqn:term_eta}).

\end{proof}

\begin{prop}
The penalty of estimating the correlation from measurements monotonically decreases when the measurement rate $K/N$ increases. It further becomes negligible at high measurement rate.
\end{prop}

\begin{proof}
In Proposition ~\ref{prop:one} we have shown that the difference between the data cost functions estimated from compressed measurements $E_d(\mathcal{M})$ and images $\tilde{E}_d(\mathcal{M})$ is lower and upper bounded by errors $C_l$ and $C_u$, respectively given in Eq.~(\ref{eqn:term_Cl}) and Eq.~(\ref{eqn:term_Cu}). The error $\eta = \sum_{k=1}^{N_1} \eta_k \propto  \sum_{k=1}^{N_1} \norm{\tilde{I}_{1,k} -  I_{1,k}} $ (see Eq.~(\ref{eqn:term_eta})) decreases with increasing measurement rate because $(\phi_1^k)^T \phi_1^k$ becomes an orthogonal projection operator, and $\tilde{I}_1 = \Phi_1^T\Phi_1I_1$ becomes arbitrarily close to $I_1$ when the number of measurements increases. Therefore, the errors $C_l$ and $C_u$ decrease as the measurement rate increases and when sufficient number of measurements are taken the errors $C_l$ and $C_u$ become negligible, i.e., $E_d(\mathcal{M}) \approx \tilde{E}_d(\mathcal{M})$. 
\end{proof}

Due to the error between the cost functions $E_d(\mathcal{M})$ and $\tilde{E}_d(\mathcal{M})$, the solution $\mathcal{M}$ estimated from the linear measurements is not accurate especially at low measurement rates. The solution of the correlation estimation problem in the compressed domain might thus be quite far from the actual correlation between images. However, as the number of measurements increases the approximation in the compressed domain becomes more accurate and the solution of the correlation estimation tends to the actual correlation between original images. 

\section{Experimental results} \label{sec:results}

\subsection{Setup}

We analyze the performance of the correlation estimation in both stereo and video imaging applications. The random projections are computed using a scrambled Fourier measurement matrix where the scrambled operator is a diagonal matrix with entries $\pm 1$ taken from an i.i.d. Bernoulli random variable with equal probability \cite{Gan_JLlemma}. In all our experiments, we sample both images using a same measurement rate. The correlation is estimated without prior reconstruction of images by minimizing the objective function in Eq.~(\ref{eqn:final_energy}).

For the stereo imaging case, we evaluate the disparity estimation performance in two natural image sets namely \emph{Tsukuba} and \emph{Venus}\footnote{Available in http://vision.middlebury.edu/stereo/data/scenes2001/}  \cite{Scharstein}.  These datasets have been captured by a camera array where the different viewpoints correspond to translating the camera along one of the image coordinate axis. In such a scenario the motion of objects due to the viewpoint  change is restricted to the horizontal direction with no motion along the vertical direction. The disparity estimation is thus a one-dimensional search problem, and the smoothness and data cost functions are modified accordingly by assuming that $\bold{m}^v = 0$. The size of the search windows used in our experiments are 16 pixels for Tsukuba and 20 pixels for Venus. In our experiments we estimate disparity in both dense (per pixel) and block settings, where the block size is fixed to $4\times 4$ pixels. 

In the video scenario, we analyze the motion estimation accuracy in two synthetic scenes, namely \emph{Yosemite} and \emph{Grove},  and one natural scene \emph{Mequon}\footnote{Available in http://vision.middlebury.edu/flow}. The Grove and Mequon datasets are resampled to a resolution of $160 \times 120$ pixels using bilinear filters. The size of search windows is of $\pm3$ pixels in both horizontal and vertical directions. For the sake of simplicity, we estimate a motion field for blocks of pixels with size of $4\times 4$ pixels and not motion vectors for each pixel.

The accuracy of the correlation estimation is evaluated by comparing to groundtruth information and to the correlation estimated from the reconstructed images. We propose another representation of the accuracy of the correlation by discussing the quality of the second view $\hat{I}_2$ that is reconstructed by prediction of the first frame according to estimated correlation. We then analyze the influence of the sampling matrix and the effect of measurement quantization on the correlation estimation performance. We finally show the importance of accurate correlation estimation in a novel joint reconstruction algorithm where images are reconstructed from measurements while satisfying sparsity constraints as well as consistency with the correlation information. Note that in practice, the groundtruth correlation model and the original images are not available a priori to estimate an optimal regularization parameter $\lambda$. In such cases, the regularization parameter $\lambda$ can be estimated based on learning from a set of training images or using the automated method proposed in \cite{GC_regparam}.  In our experiments, we however select the parameter $\lambda$ based on trial and error experiments. 

\subsection{Disparity estimation performance} \label{results:disp}

We first illustrate disparity maps for the Venus dataset where the compressed data have been obtained with a different measurement matrix for each image, i.e., $\Phi_1 \neq \Phi_2$. Fig.~\ref{Fig:venu_deptherror}(b) and Fig.~\ref{Fig:venu_deptherror}(d) show the disparity map from a measurement rate $0.2$ and $0.7$ respectively. Fig.~\ref{Fig:venu_deptherror}(c) and Fig.~\ref{Fig:venu_deptherror}(e) represent the corresponding disparity errors. Comparing the results with the groundtruth given in Fig.~\ref{Fig:venu_deptherror}(a) we see that at low measurement rate (corresponding to $0.2$) we estimate a coarse version of the disparity map. Quantitatively the disparity error with respect to the groundtruth is found out to be $41\%$ when measured as the percentage of pixels where the absolute error is greater than one \cite{Scharstein} as shown in Fig.~\ref{Fig:venu_deptherror}(c). At higher rate, the disparity map is more accurate and the disparity error drops below $11\%$.

  \begin{figure*}[h!]
  \centering
 $\begin{array}{@{\hspace{-0.25 in}} c@{\hspace{-0.25 in}}c @{\hspace{-0.25 in}} c@{\hspace{-0.25 in}} c@{\hspace{-0.25 in}} c} \multicolumn{1}{l}{\mbox{}} &  \multicolumn{1}{l}{\mbox{}} &\multicolumn{1}{l}{\mbox{}} &  \multicolumn{1}{l}{\mbox{}} &\multicolumn{1}{l}{\mbox{}} \\
   \epsfxsize=1.6in \epsffile{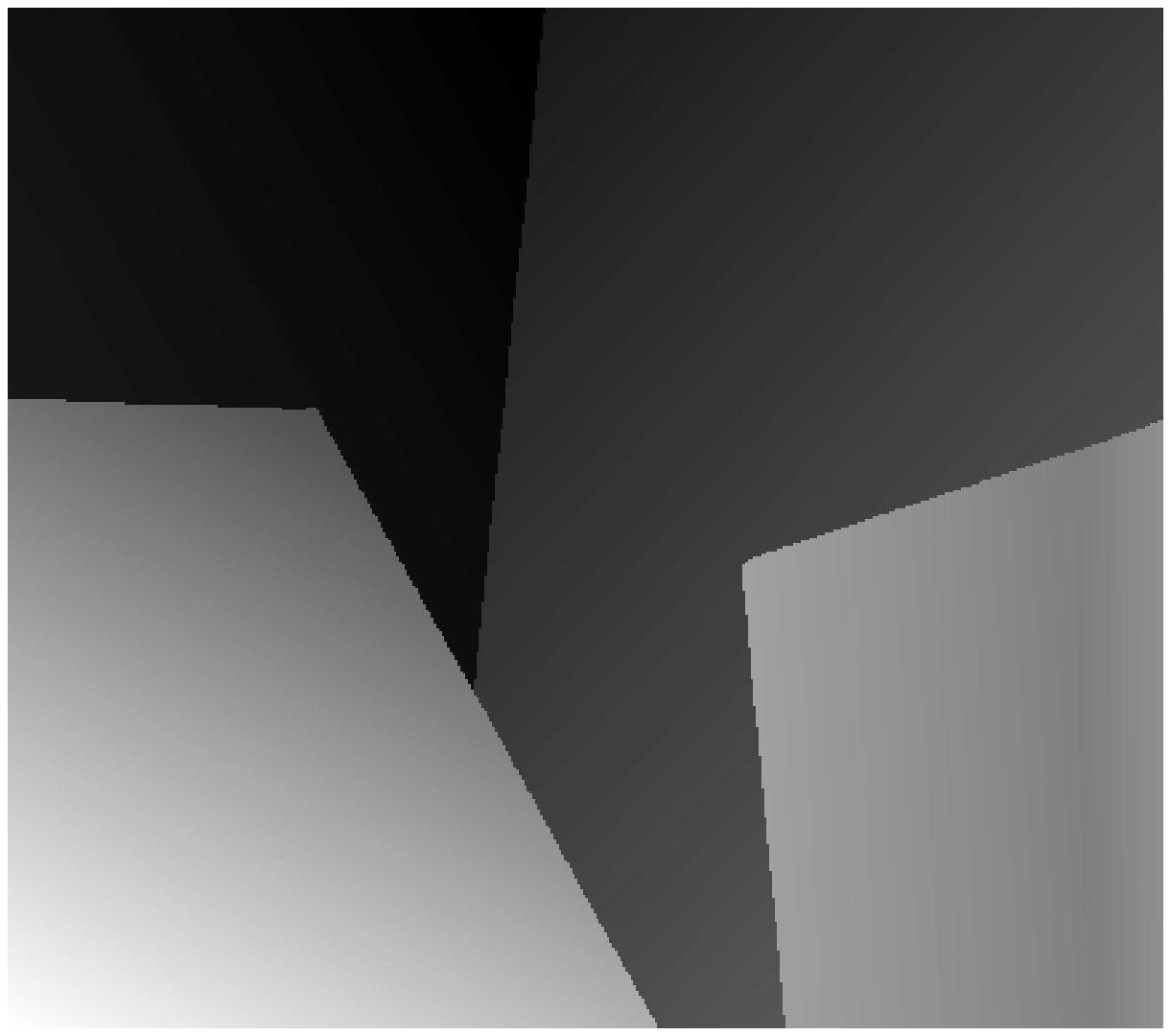} & \epsfxsize=1.6in \epsffile{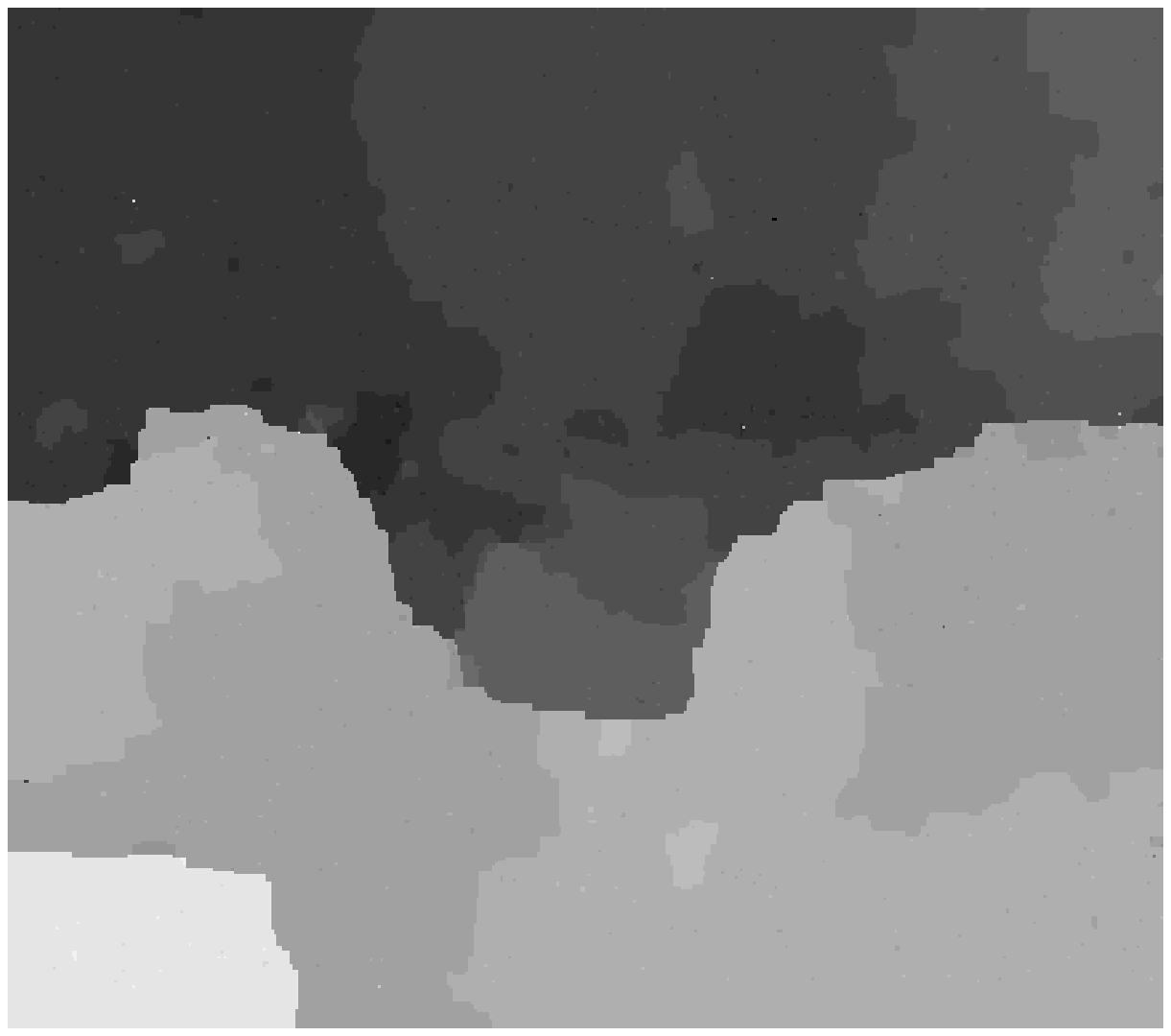} & \epsfxsize=1.6in \epsffile{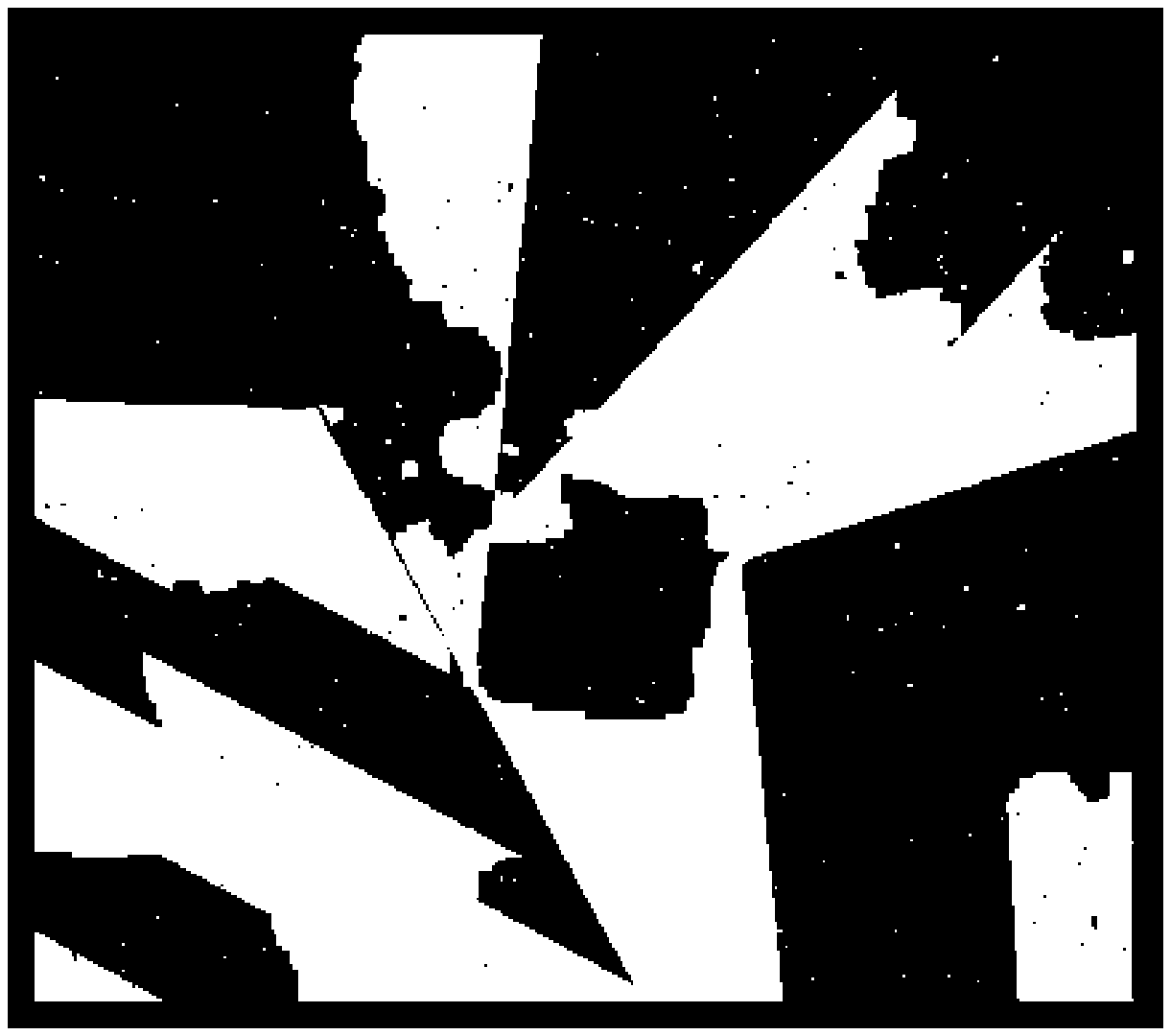} & \epsfxsize=1.6in \epsffile{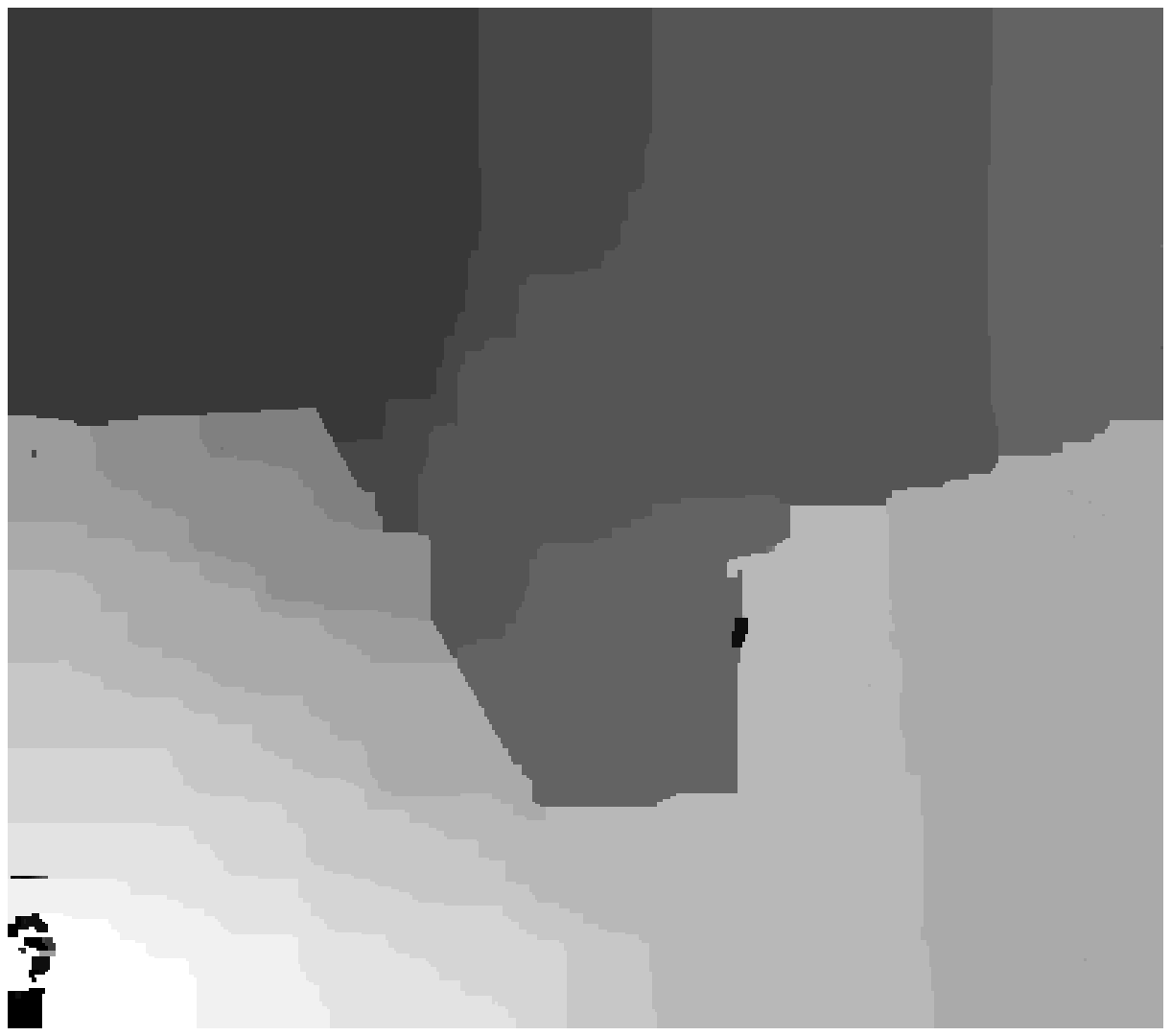} & \epsfxsize=1.6in \epsffile{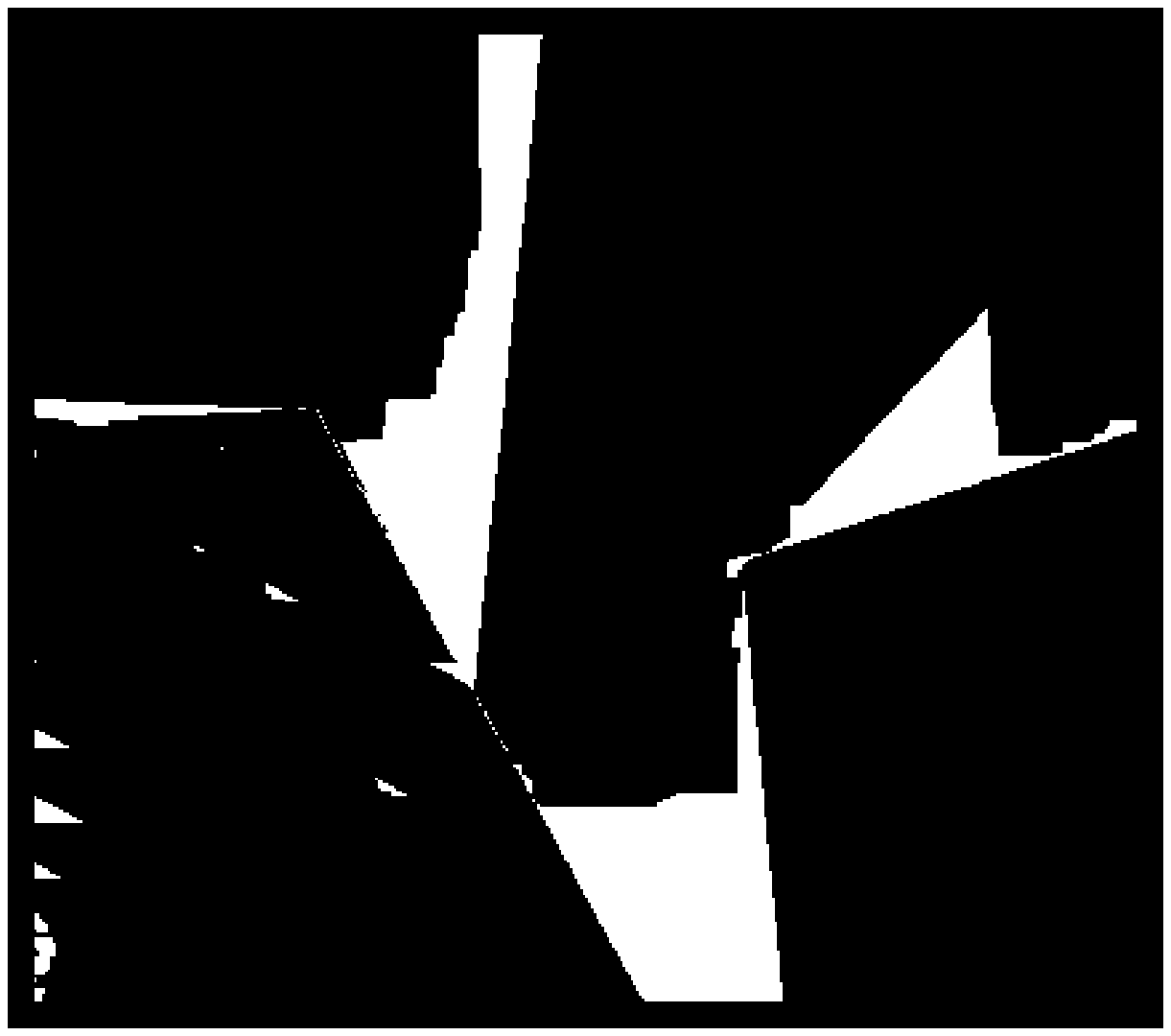} \vspace{-0.2in} \\
   \mbox{(a) $\bold{M}^h $ } & \mbox{(b) $\bold{m}^h $} & \mbox{(c) $|\bold{M}^h-\bold{m}^h| >1$} & \mbox{(d) $\bold{m}^h$} & \mbox{(e) $|\bold{M}^h-\bold{m}^h| >1$}
  \end{array}$
 \caption{Comparison of the estimated disparity image with respect to groundtruth information at measurement rates 0.2 and 0.7 in the Venus dataset. (a) Groundtruth disparity image $\bold{M}^h$; (b) computed dense disparity image $\bold{m}^h$ at measurement rate $0.2$; (c) disparity error at rate 0.2.  The pixels with absolute error greater than one is marked in white. The percentage of white pixels is $41\%$. (d) Computed dense disparity image $\bold{m}^h$ at measurement rate $0.7$; (e) disparity error at rate 0.7. The percentage of white pixels is $10.7 \%$. }
  \label{Fig:venu_deptherror}
  \end{figure*}
  

  \begin{figure*}[h!]
  \centering
 $\begin{array}{@{\hspace{-0.2 in}} c@{\hspace{-0.2 in}}c @{\hspace{-0.2 in}} c@{\hspace{-0.2 in}} c} \multicolumn{1}{l}{\mbox{}} &  \multicolumn{1}{l}{\mbox{}} &\multicolumn{1}{l}{\mbox{}} &  \multicolumn{1}{l}{\mbox{}} \\
  \epsfxsize=1.8in {\epsffile{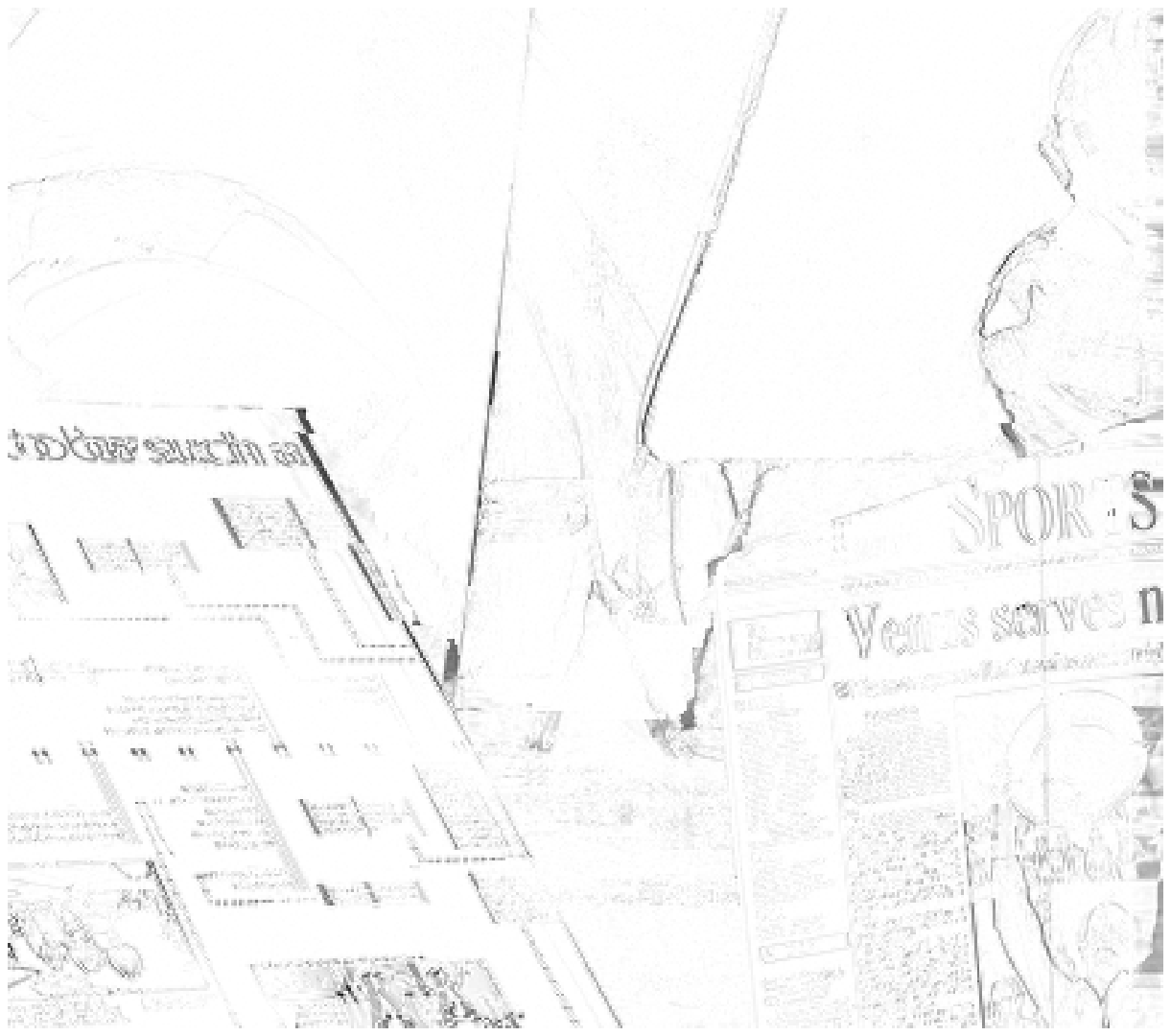}} & \epsfxsize=1.8in {\epsffile{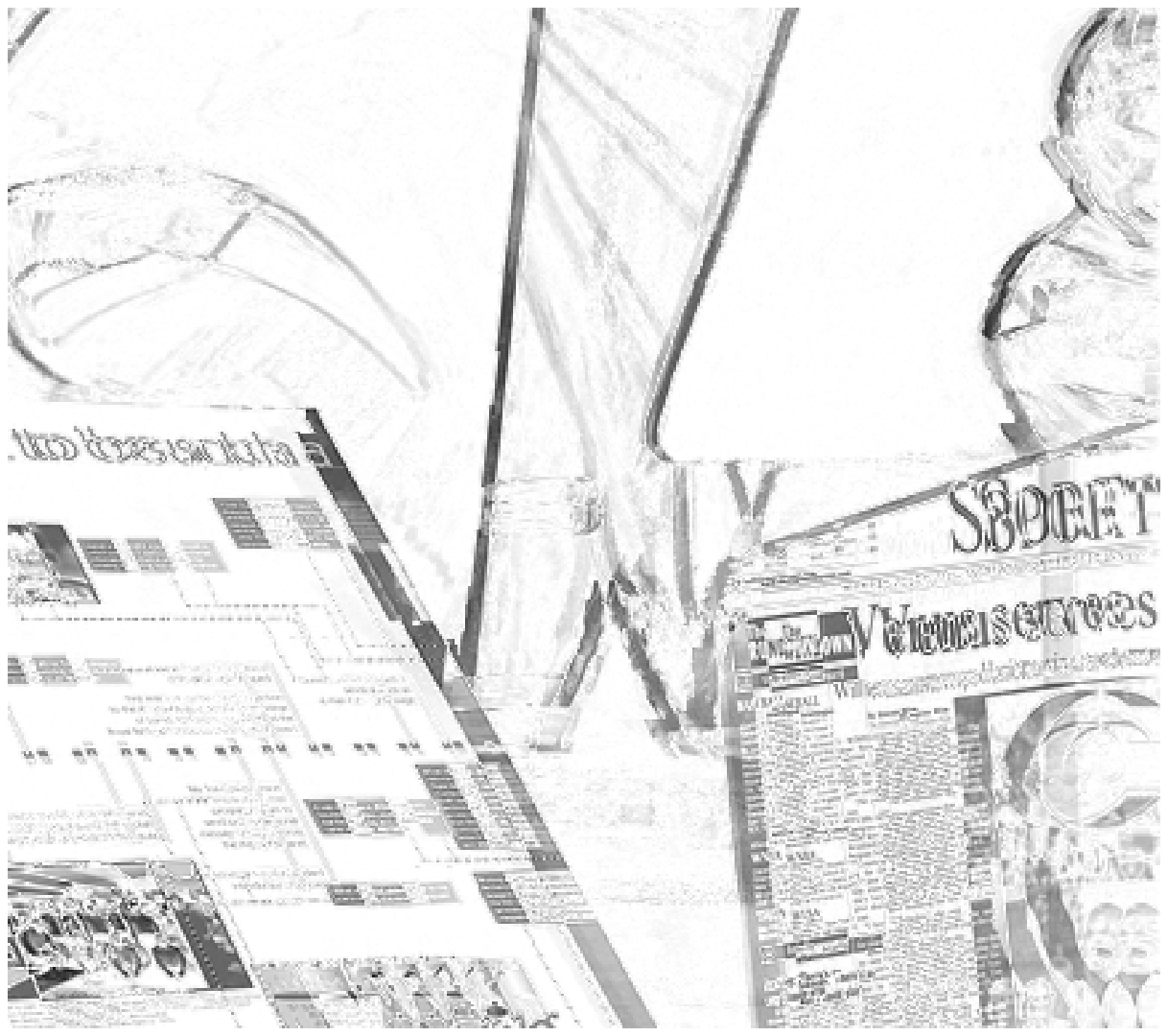}} &  \epsfxsize=1.8in {\epsffile{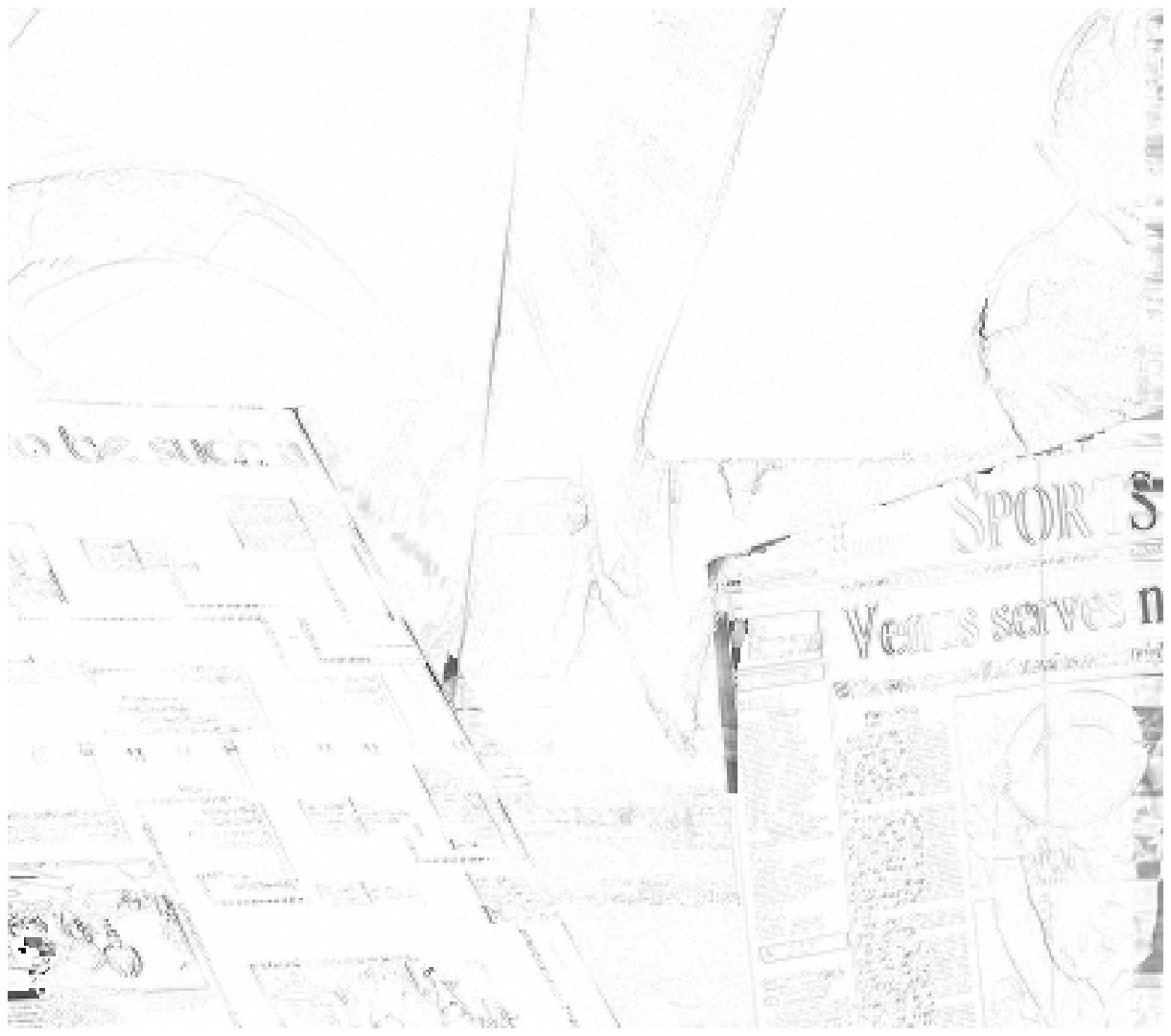}} & \epsfxsize=1.8in {\epsffile{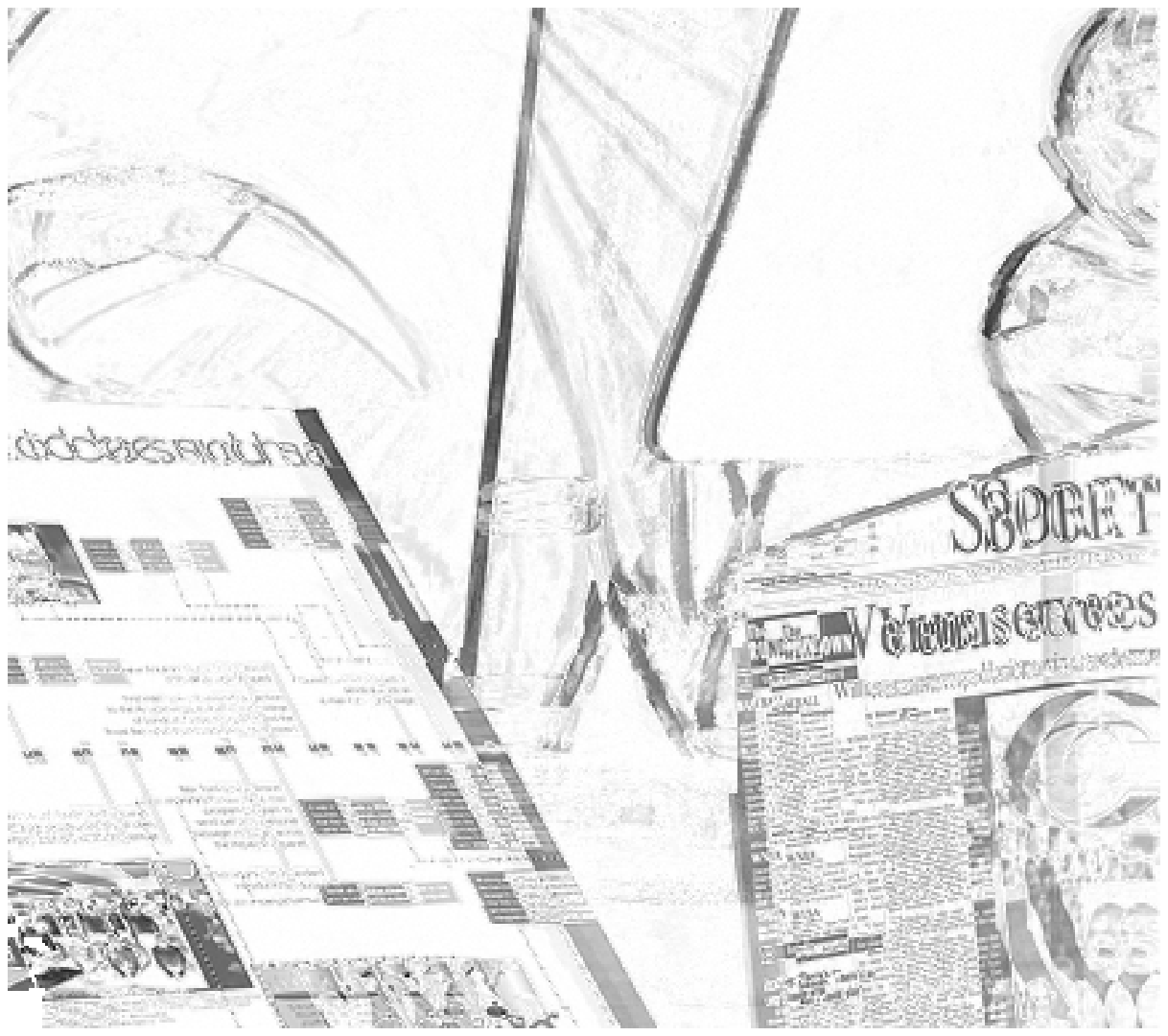}} \vspace{-0.2in} \\
  \mbox{(a) MSE: 205 } & \mbox{(b) MSE: 1221 } &  \mbox{(c)  MSE: 101 } & \mbox{(d) MSE: 1219}
  \end{array}$
\caption{Evaluating the accuracy of disparity image in Fig.~\protect \ref{Fig:venu_deptherror}(b) and Fig.~\protect \ref{Fig:venu_deptherror}(d) in terms of image prediction quality for the Venus dataset. The disparity images in Fig.~\protect \ref{Fig:venu_deptherror}(b) and Fig.~\protect \ref{Fig:venu_deptherror}(d) are used to predict the image $\hat{I}_2$ at measurement rates 0.2 and 0.7 respectively.  (a) Inverse prediction $1-|\hat{I}_2-I_2|$  at  a measurement rate of $0.2$; (b) inverse prediction $1-|\hat{I}_2-I_1|$  at  a measurement rate of $0.2$; (c) inverse prediction $1-|\hat{I}_2-I_2|$  at  a measurement rate of $0.7$; (d) inverse prediction $1-|\hat{I}_2-I_1|$  at  a measurement rate of $0.7$. 
The error is inverted, so that the white pixels correspond to no error. }
 \label{Fig:venu_rquality_0.2}
 \end{figure*}

We then show the quality of the reconstruction of the second image that is predicted from the first image using the correlation estimate. When a coarse disparity map $\bold{m}^h$ (i.e., estimated at low measurement rate) is used for image prediction the resulting predicted image $\hat{I}_2$ is closer to $I_2$ than $I_1$ (see Fig.~\ref{Fig:venu_rquality_0.2}(a) and  Fig.~\ref{Fig:venu_rquality_0.2}(b)  respectively). We observe that the mean square error (MSE) between the predicted image $\hat{I}_2$ and $I_2$ is smaller than the error between $\hat{I}_2$ and $I_1$, which confirms the benefit of the disparity estimate in the prediction. When the measurement rate increases the quality of the disparity map improves and the quality of the predicted image $\hat{I}_2$ also improves substantially as it can be observed in Fig.~\ref{Fig:venu_rquality_0.2}(c)  and Fig.~\ref{Fig:venu_rquality_0.2}(d) where the measurement rate is set to $0.7$.

\begin{figure}[t]
\centering
    \epsfxsize=2.8in \epsffile{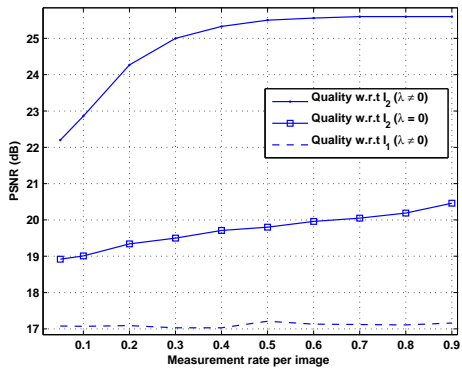}
 \caption{Comparison of the quality of predicted image $\hat{I}_2$ with respect to $I_2$ and $I_1$ with and without regularization, i.e., $\lambda \neq 0$ and $\lambda = 0$ in Eq.~(\protect \ref{eqn:final_energy}) respectively in the Tsukuba dataset. The image prediction is carried out using dense disparity image. }
  \label{Fig:effe_Rend}
  \end{figure}

The next experiments illustrate the benefit of regularization in the disparity estimation problem. Fig.~\ref{Fig:effe_Rend} plots the quality of the predicted image  $\hat{I}_2$ with and without smoothness cost (i.e., $\lambda \neq 0$ and $\lambda = 0$ respectively in Eq.~(\ref{eqn:final_energy}) respectively) for the Tsukuba dataset. It is clear that the quality of $\hat{I}_2$ is improved by enabling the regularization term in our optimization framework. Similar experimental finding is observed for the Venus dataset.  From Fig.~\ref{Fig:effe_Rend} we further observe that the quality of the predicted image $\hat{I}_2$ at a measurement rate $0.05$ is $22.2$ dB (the corresponding disparity error is 39\%), which is approximately $3.5$ dB away from the saturation point or from global minima solution due to influence of the penalty terms $C_l$ and $C_u$ discussed in Section \ref{sec:optimal_property}. As  the measurement rate increases, the influence of the terms $C_l$ and $C_u$ decreases. As a result the quality of predicted image $\hat{I}_2$ increases with the measurement rate and saturates above rates $> 0.5$. In other words, our scheme gives optimal disparity solution at high measurement rate. We also carry out experiments using the same measurement matrix for both images, i.e., $\Phi_1= \Phi_2$. Fig.~\ref{Fig:tsu_prop_indep}(a) and Fig.~\ref{Fig:tsu_prop_indep}(b) compare the PSNR quality of the predicted image $\hat{I}_2$ and the disparity error (DE), respectively with the results obtained with different measurement matrices. It is clear that the prediction image quality and the disparity accuracy improve when different measurement matrices are used as this brings more information from both images to solve the correspondence problem. 

We finally compare our disparity estimation results to a scheme that first reconstructs the images before estimating the disparity map. The images are reconstructed independently from the corresponding measurements by solving a convex optimization problem. We denote this methodology as \emph{disparity from reconstructed images} (DFR). We have tried out two different reconstruction methodologies: (1) DFR-sparsity that consists in minimizing the $l_1$ norm of the sparse coefficients assuming that the image is sparse in a particular orthonormal basis (e.g., a wavelet basis); this problem is solved using GPSR \cite{GPSR}; (2)  DFR-TV that minimizes the TV norm of the reconstructed image; this problem is solved using BPDQ toolbox \cite{bpdqtbx}. The disparity map is then estimated using $\alpha$-expansion mode in Graph Cuts applied on the reconstructed images. Fig.~\ref{Fig:tsu_prop_indep}  shows the comparison of the proposed scheme with the DFR-sparsity and DFR-TV schemes for the Tsukuba dataset. From Fig.~\ref{Fig:tsu_prop_indep}(a) and Fig.~\ref{Fig:tsu_prop_indep}(b) we observe that the performance of our low complexity correlation solution competes with DFR-sparsity scheme. Especially at rates smaller than $0.1$, our scheme performs better than the DFR-sparsity scheme as the poor image reconstruction quality in the DFR-sparsity scheme leads to a bad estimation of disparity map. On the other hand, DFR-TV scheme significantly outperforms our scheme at lower rates due to good reconstruction image quality. However, our scheme estimates the optimal correlation for rates above $0.5$ and thus performs similar to the DFR-TV scheme at high measurement rate but with a complexity that is dramatically smaller as it avoids image reconstruction. In particular, in our experiments we have observed that the running time of Graph Cuts algorithms that estimate the correlation information from linear measurements is approximately same as the one that estimates the correlation information from reconstructed images. The complexity of our correlation estimation scheme stays reasonable due to the efficiency of Graph Cuts algorithms whose complexity is bounded by a low order polynomial \cite{Boykov_GC}. Comparing to the DFR-sparsity and DFR-TV schemes, we save on the complexity corresponding to solving the $l_2$-$l_1$ and $l_2$-{TV} optimization problems respectively. It is however hard to precisely give the order of complexity of solving the $l_2$-$l_1$ and $l_2$-{TV} optimization problems, as it is highly depend on the type of solvers. For some of the popular solvers like GPSR \cite{GPSR} and NESTA \cite{nesta}, the order of complexity is given as $T\mathcal{O}(NlogN)$, where $T$ is the number of iterations and $N$ is the resolution of the image \cite{Montager}. Therefore, comparing to the DFR schemes we  save a complexity of $2T\mathcal{O}(NlogN)$.   


The complexity of the proposed scheme can be further reduced when a disparity value is estimated per block instead of per pixel. Fig.~\ref{Fig:tsu_prop_indep_blk} compares the performance of our scheme with respect to the DFR schemes when disparity is estimated using per block with block size $4 \times 4$. Comparing Fig.~\ref{Fig:tsu_prop_indep_blk} and  Fig.~\ref{Fig:tsu_prop_indep}(a) we see that the relative performances between the schemes remains approximately the same when the disparity image is estimated per block or per pixel. This confirms that the proposed scheme can easily adapt the granularity of disparity estimation without big penalty in order to meet the complexity requirements at the decoder.

\begin{figure*}
\centering
 $\begin{array}{c@{\hspace{0 in}}c@{\hspace{0 in}}} \multicolumn{1}{l}{\mbox{}} &  \multicolumn{1}{l}{\mbox{}}  \\
    \epsfxsize=2.9in \epsffile{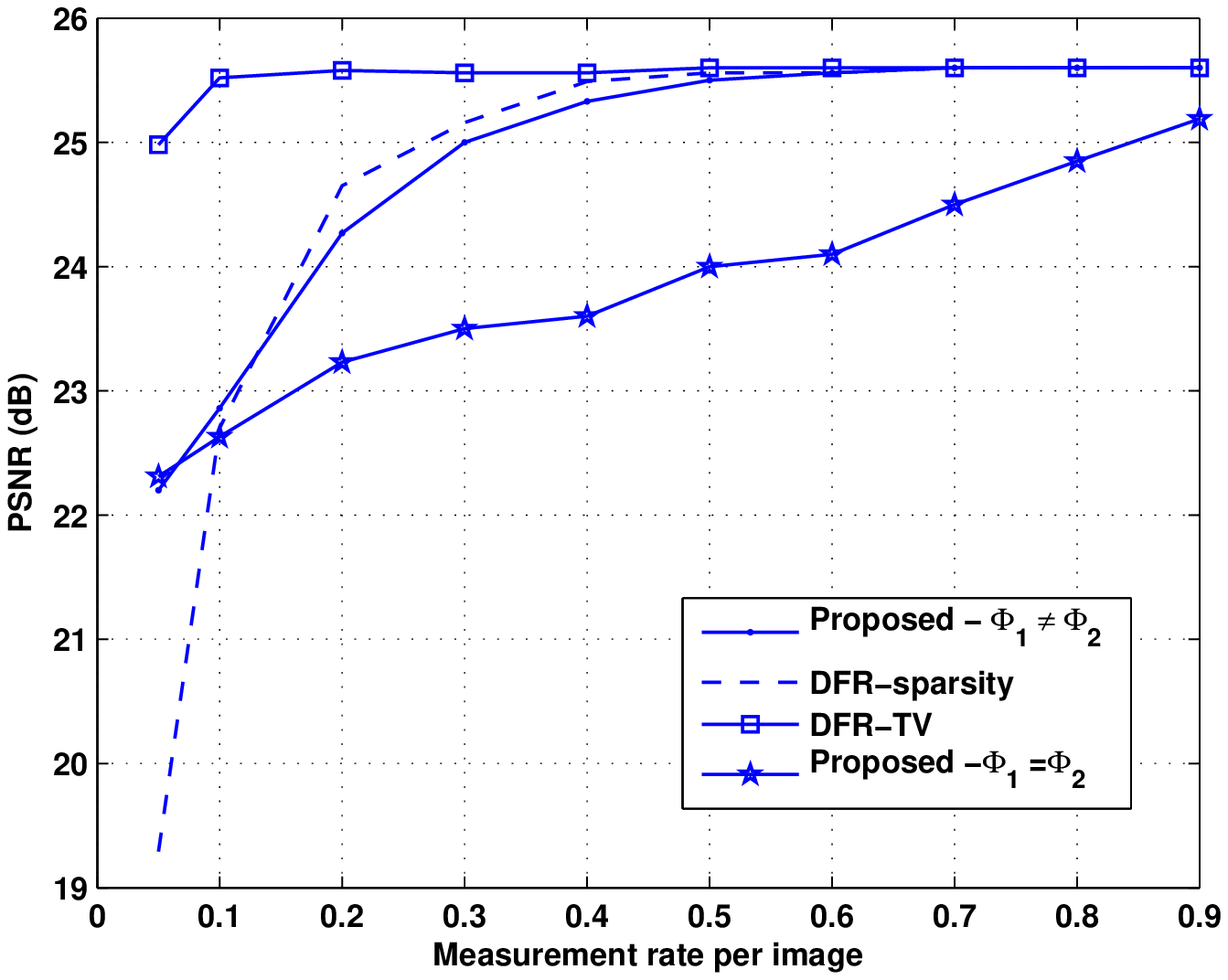}
             & \epsfxsize=2.9in \epsffile{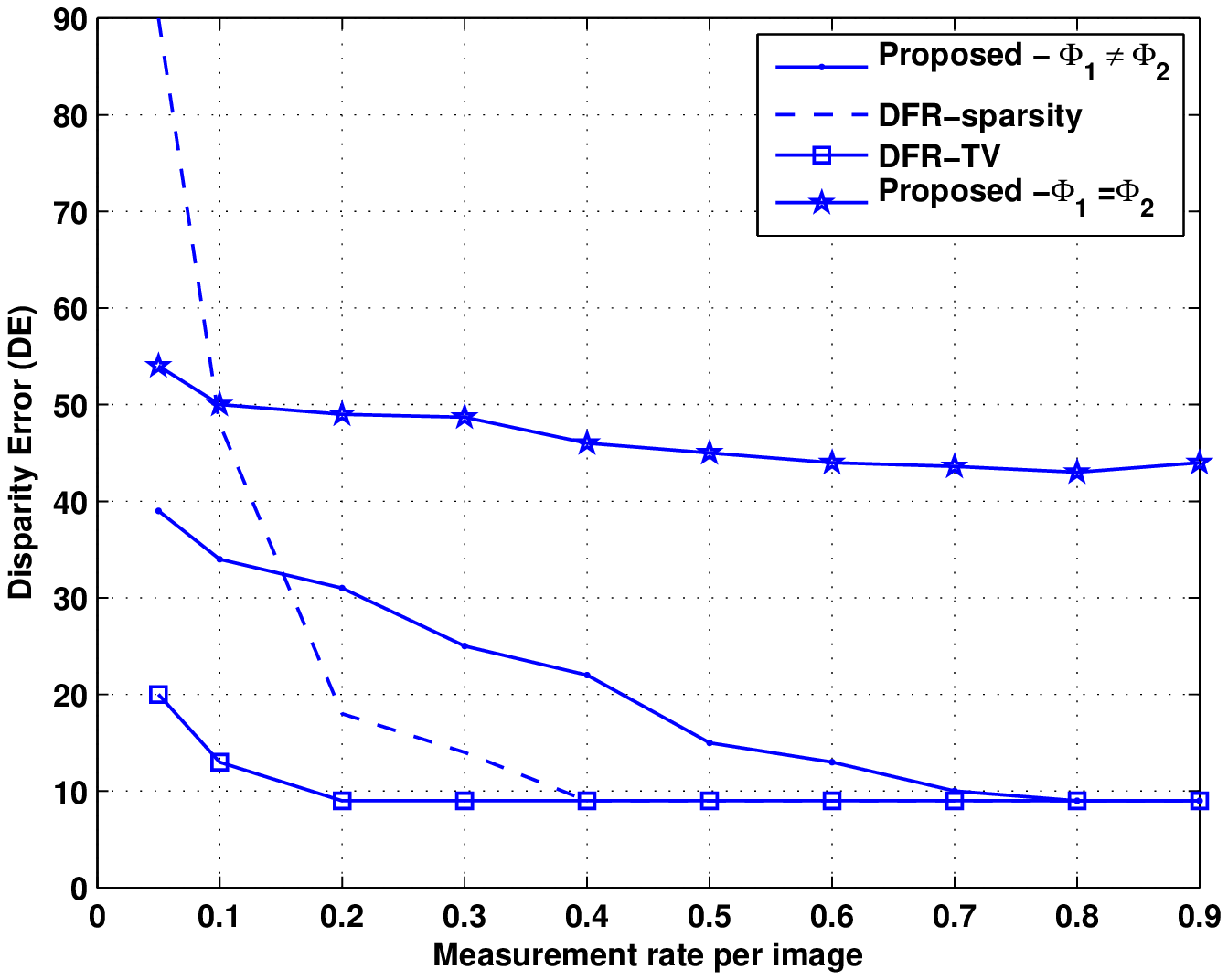} \\
 \mbox{(a)  } & \mbox{(b) }
   \end{array}$
 \caption{Comparison of the proposed scheme with the DFR schemes in the Tsukuba dataset: (a) comparison in terms of image prediction quality; (b) comparison in terms of disparity error. The performance of the proposed scheme is evaluated using both the same and different sets of measurement matrices.}
  \label{Fig:tsu_prop_indep}
  \end{figure*}

\begin{figure}[t!]
\centering
    \epsfxsize=2.9in \epsffile{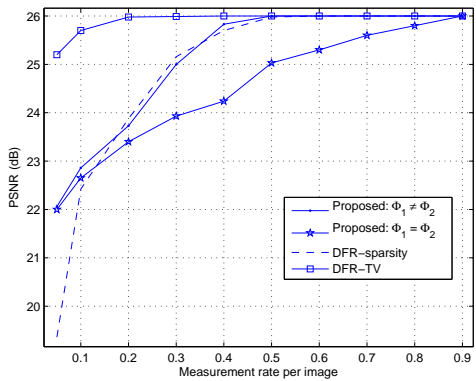}
 \caption{Comparison of the proposed scheme over disparity from reconstructed image (DFR) schemes for Tsukuba dataset. The disparity is estimated per block with a block size $4 \times 4$.  The performance of the proposed scheme is evaluated using both the same and different sets of measurement matrices. }
  \label{Fig:tsu_prop_indep_blk}
  \end{figure}

\subsection{Motion estimation performance} \label{result:motion}

\begin{figure}[t]
\centering
    \epsfxsize=2.9in \epsffile{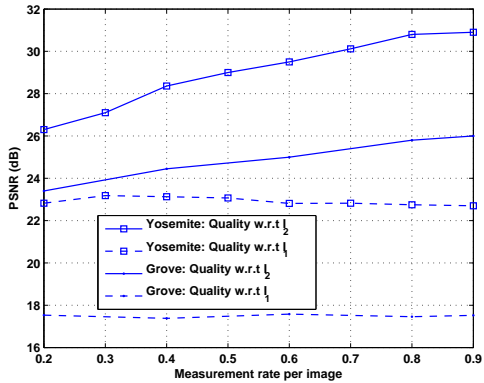}
 \caption{Illustration of the accuracy of motion field estimation in terms of image prediction quality for the Yosemite and Grove datasets. The quality of the predicted image $\hat{I}_2$ is compared with respect to $I_2$ and $I_1$. The prediction is carried out using the motion field estimated with block of pixels $4 \times 4$. }
  \label{Fig:motion_prediction_quality}
  \end{figure}

\begin{figure*}[!h]
\centering
 $\begin{array}{c@{\hspace{0 in}}c@{\hspace{0 in}}} \multicolumn{1}{l}{\mbox{}} &  \multicolumn{1}{l}{\mbox{}}  \\
    \epsfxsize=2.9in \epsffile{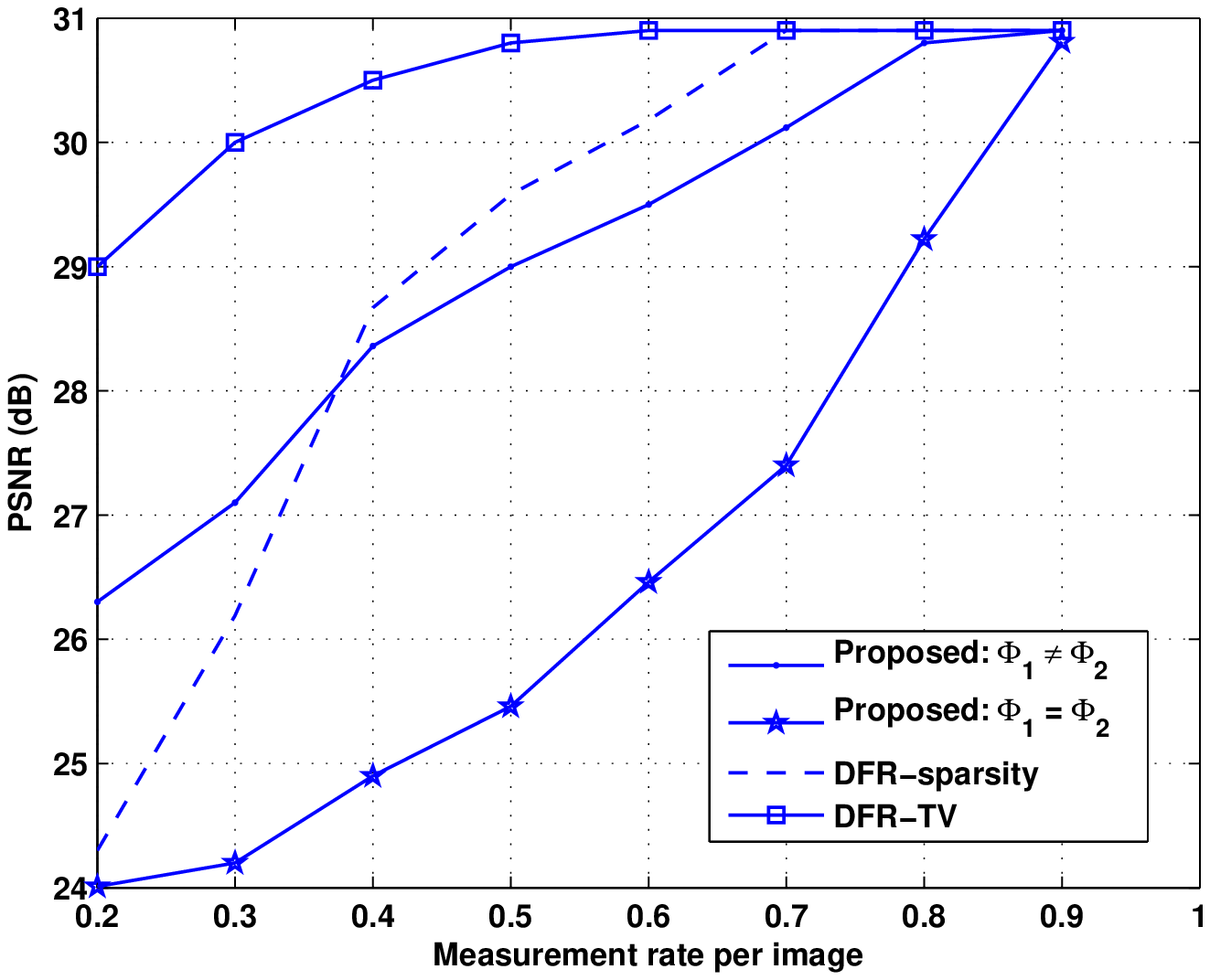}
             & \epsfxsize=2.9in \epsffile{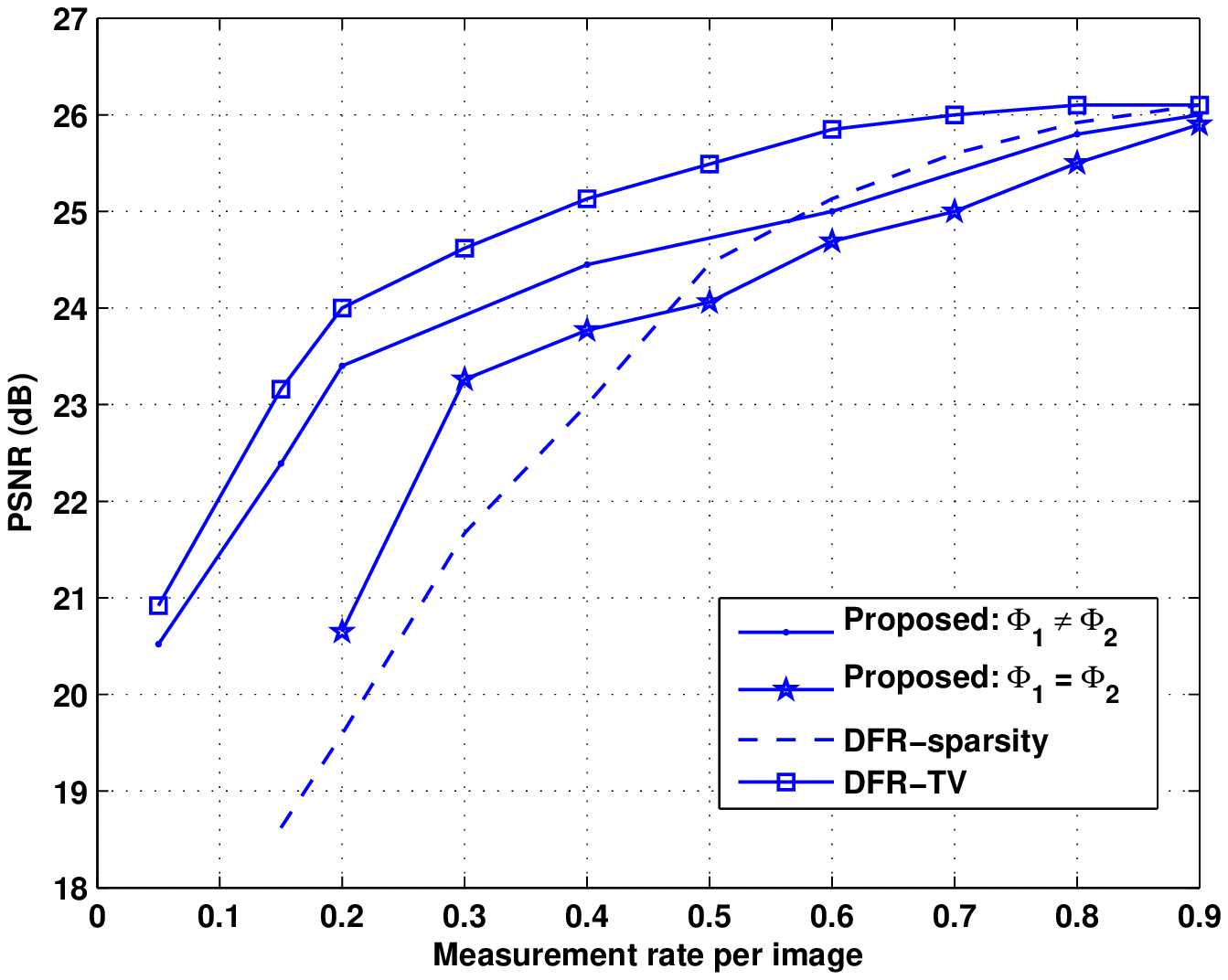} \\
     \mbox{(a)  } & \mbox{(b) }
  \end{array}$
 \caption{ Comparison of the quality of the predicted image $\tilde{I}_2$ between the proposed, DFR-sparsity and DFR-TV schemes: (a) Yosemite dataset; (b) Grove dataset. The image prediction is carried out using the motion field that is estimated using block of pixels $4 \times 4$.  }
  \label{Fig:motion_same_vs_diff1}
  \end{figure*}

We now illustrate the performance of our correlation estimation algorithm in video sequences. We first estimate motion vectors per blocks of $4\times4$ pixels with different sensing matrices $\Phi_1 \neq \Phi_2$ for each image. These vectors are then used to predict the second image from the first image. Fig.~\ref{Fig:motion_prediction_quality} compares the predicted image $\hat{I}_2$ with the original images $I_2$ and $I_1$ for Yosemite and Grove datasets respectively. It is clear that for a given measurement rate the predicted image $\hat{I}_2$ is closer to $I_2$ than $I_1$ which indicates that the motion between the images is efficiently captured by our correlation estimation algorithm. Similar experimental results are observed in the Mequon dataset. 

We then highlight the benefit of sampling the images with different sets of measurement matrices in Fig.~\ref{Fig:motion_same_vs_diff1}. From Fig.~\ref{Fig:motion_same_vs_diff1} we see that the quality of the predicted image $\hat{I}_2$ is better when the images are sampled with different measurement matrices, compared to the case where the same sampling matrix is used for all images. This confirms the results shown for the disparity estimation performance. We finally compare our results to the DFR-sparsity and DFR-TV schemes that build the correlation model from (independently) reconstructed images based on minimizing the sparsity and TV priori respectively. Fig.~\ref{Fig:motion_same_vs_diff1}(a) and Fig.~\ref{Fig:motion_same_vs_diff1}(b)  show the comparison for the Yosemite and Grove datasets respectively. From Fig.~\ref{Fig:motion_same_vs_diff1} we see that for both datasets the proposed scheme performs better than the DFR-sparsity scheme at low rates and competes with the DFR-sparsity scheme at high rates, as observed in the disparity estimation study. Similar experimental findings are observed in the Mequon dataset. Furthermore, we see that the proposed scheme competes with the performance of  DFR-TV scheme at low rate for the Grove dataset (see Fig.~\ref{Fig:motion_same_vs_diff1}(b)); as the Grove scene is textured with limited low frequency components, the $TV$ prior in the reconstruction scheme results in poor reconstruction quality in the textured areas. Overall, the proposed scheme provides effective motion estimation results, while avoiding an order of computational complexity $2T\mathcal{O}(NlogN)$ involved in the image reconstruction steps with the DFR-TV and DFR-sparsity schemes.

\subsection{Measurements quantization}

We briefly study here the performance of the correlation estimation algorithm when the measurements are affected by noise and in particular quantization noise. For the sake of simplicity we quantize the measurements using a uniform quantizer and we denote the quantized measurements as $\hat{Y}_1$ and $\hat{Y}_2$. We then estimate the correlation model by minimizing the energy in Eq.~(\ref{eqn:final_energy}) using the quantized measurements $\hat{Y}_1$ and $\hat{Y}_2$. 

Fig.~\ref{Fig:venu_effectQuant} shows the effect of quantization on the disparity estimation performance when measurements are uniformly quantized using $2$-, $3$- and $4$-bits. We use the disparity map to predict the second image in the Venus dataset. Interestingly we see that the $4$-bit quantizer does not significantly affect the quality of the disparity image, as the degradation hardly reaches $0.5$ dB in the quality of $\hat{I}_2$ at low to medium rates. As expected however the quality of the predicted image $\hat{I}_2$ decreases with increasing quantizer coarseness level for a fixed measurement rate. We then compare our results with the DFR-TV scheme that reconstructs the images by solving an optimization problem based on $BPDN_p$  in order to efficiently handle the quantization noise \cite{Laurent_BPDNp} and then use the reconstructed images for disparity estimation. From  Fig.~\ref{Fig:venu_effectQuant} we see that the performance gap between the DFR-TV and compressed domain estimation is approximately the same in both the unquantized and quantized (i.e., 2-bit) scenarios.  However, it should be noted that the DFR-TV scheme considers the nonlinearities due to quantization while reconstructing the images. Such effects are not considered in the proposed scheme. The solution of our scheme could also be improved by considering the quantization non-linearities but this problem is beyond the scope of this paper. 

\begin{figure}[t]
\centering
    \epsfxsize=2.9in \epsffile{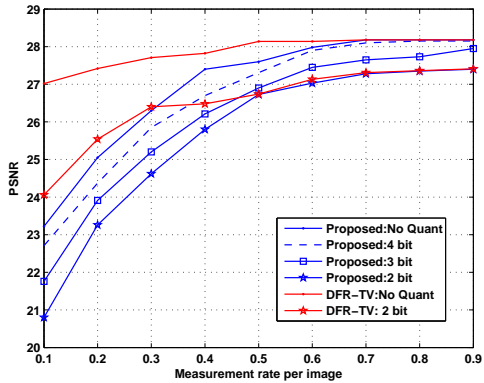}
 \caption{Effect of quantization on the quality of the predicted image $\hat{I}_2$ in the Venus dataset. When the measurements are quantized we used $BPDN_p$ \cite{Laurent_BPDNp} to reconstruct the images in DFR-TV based scheme. The image prediction is carried out using the dense disparity image estimated by solving \protect{Eq.~(\ref{eqn:final_energy})}. }
  \label{Fig:venu_effectQuant}
  \end{figure}

We also study the effect of  measurement quantization on the quality of the motion field. Fig. ~\ref{Fig:effect_quant1}(a)  and Fig. ~\ref{Fig:effect_quant1}(b) shows the quality of the predicted image where the prediction is performed with motion vectors estimated from quantized measurements. The quality of the predicted image or equivalently the accuracy of motion estimation is reduced when the measurements are quantized as expected. Similarly to the case of disparity estimation, the influence of quantization is negligible when the measurements are quantized with a $4$-bit quantizer. We also compare our results to the DFR-TV scheme that reconstructs the images by solving an optimization problem based on $BPDN_p$ and estimate motion from the reconstructed images. The performance of the DFR-TV scheme for the $2$-bit quantization scenario is shown in Fig. ~\ref{Fig:effect_quant1}. Interestingly, when the measurements are quantized we see that the proposed scheme competes with the DFR-TV scheme, because of the poor image reconstruction performance in the DFR-TV scheme when the measurements are coarsely quantized (i.e., 2-bit quantizer).  

\begin{figure*}[!h]
\centering
 $\begin{array}{c@{\hspace{0 in}}c@{\hspace{0 in}}} \multicolumn{1}{l}{\mbox{}} &  \multicolumn{1}{l}{\mbox{}}  \\
    \epsfxsize=2.9in \epsffile{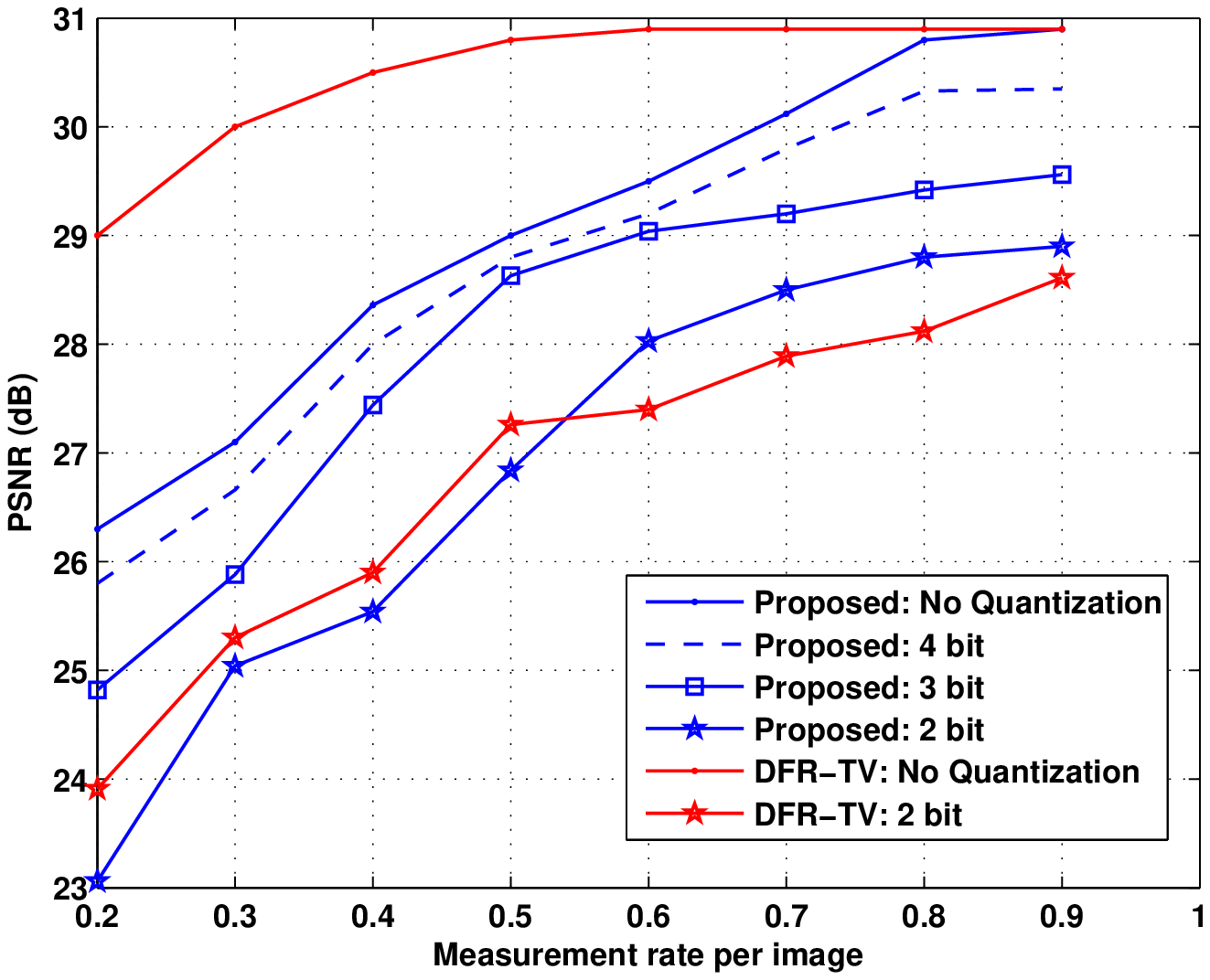}
             & \epsfxsize=2.9in \epsffile{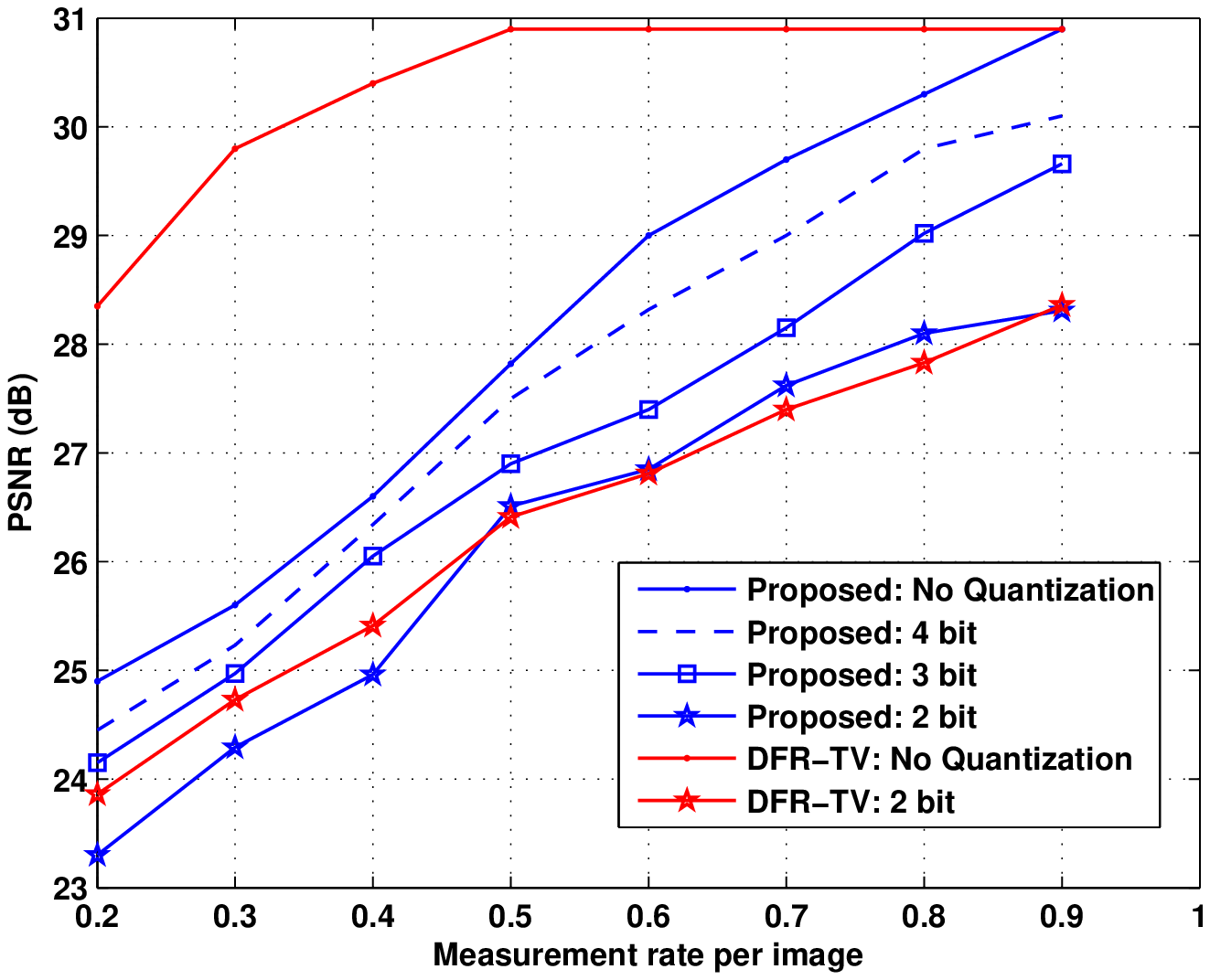} \\
   \mbox{(a) } & \mbox{(b)}
  \end{array}$
 \caption{Effect of using 4-, 3- and 2-bits quantizers on the image prediction quality for (a) Yosemite dataset and (b) Mequon dataset.  When the measurements are quantized we used $BPDN_p$  \cite{Laurent_BPDNp} to reconstruct the images in DFR-TV scheme. The image prediction is carried out using the motion field  estimated by solving \protect{Eq.~(\ref{eqn:final_energy})}}
  \label{Fig:effect_quant1}
  \end{figure*}

\subsection{Importance of correlation in joint reconstruction} \label{sec:jointrec}

We finally propose to study the importance of accurate correlation estimation in a novel joint reconstruction algorithm (see Fig.~\ref{Fig:block_scheme}). We propose to reconstruct a pair of images $\acute{I}_1$ and $\acute{I}_2$ by enforcing consistency with the compressed information and also with the estimated correlation model. A pair of image $\acute{I}_1$ and $\acute{I}_2$ is  reconstructed as a solution of the following constrained optimization problem:
\begin{equation} \label{eqn:jr}
\begin{array}{ll}
(\acute{I}_1, \acute{I}_2) = \underset{I_1, I_2}{\argmin} \;(\normc{ \psi^*{I}_1} +  \normc{\psi^*{I}_2}) \; \; \mbox{s.t.} \; \; \norm{ Y_1 - \Phi_1 {I}_1} = 0,   \norm{Y_2 - \Phi_2 {I}_2} = 0,  \sqnorm{{I}_2 - A{I}_1} \leq \epsilon,
  \end{array}
\end{equation}
where $\psi$ is a redundant dictionary or an orthonormal basis in which the image is assumed to be sparse and $\psi^*$ is the conjugate transpose of $\psi$. From Eq.~(\ref{eqn:jr}) it is clear that the images can be reconstructed independently if we solve the optimization without the last constraint $\sqnorm{{I}_2 - A {I}_1} \leq \epsilon$. This corresponds to the independent reconstruction of sparse images in $\psi$ which agrees with the measurement information. By adding the last constraint, we impose that the pair of images also fit with the correlation model, in addition to the sparsity and data fidelity constraints. As a result the reconstruction quality for a given measurement rate is better when the images are reconstructed jointly than independently. The optimization problem for joint reconstruction can be re-written with proximity operators and solved efficiently using the parallel proximal algorithm (PPXA) proposed by Combettes \emph{et al.} \cite{Combettes_prox} (see \cite{Vijay_ICIP2011} for a similar solution in an asymmetric joint decoding scheme).

\begin{figure*}
\centering
$\begin{array}{c@{\hspace{0 in}}c} \multicolumn{1}{l}{\mbox{}} &  \multicolumn{1}{l}{\mbox{}} \\
    \epsfxsize=2.9in \epsffile{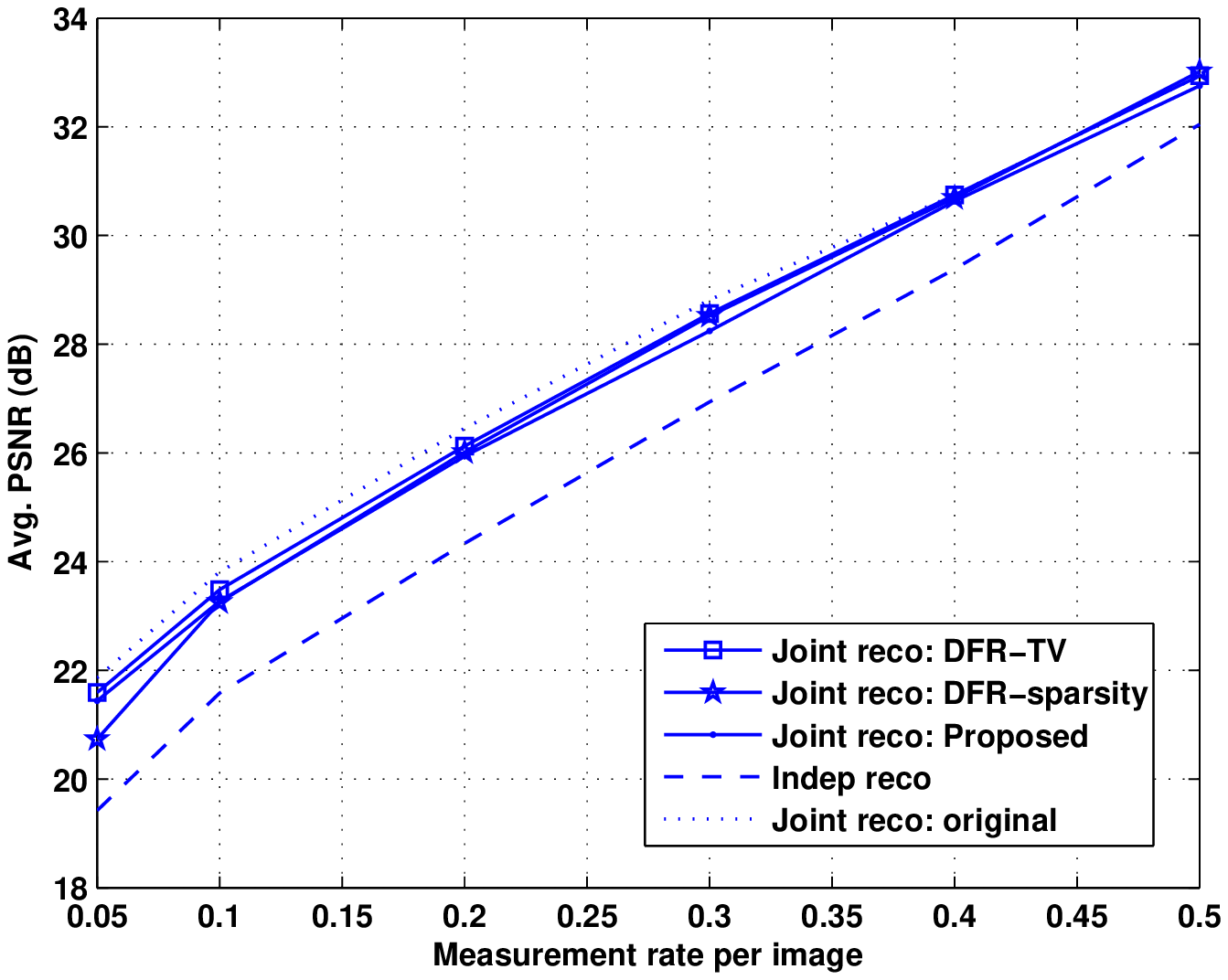}
             & \epsfxsize=2.9in \epsffile{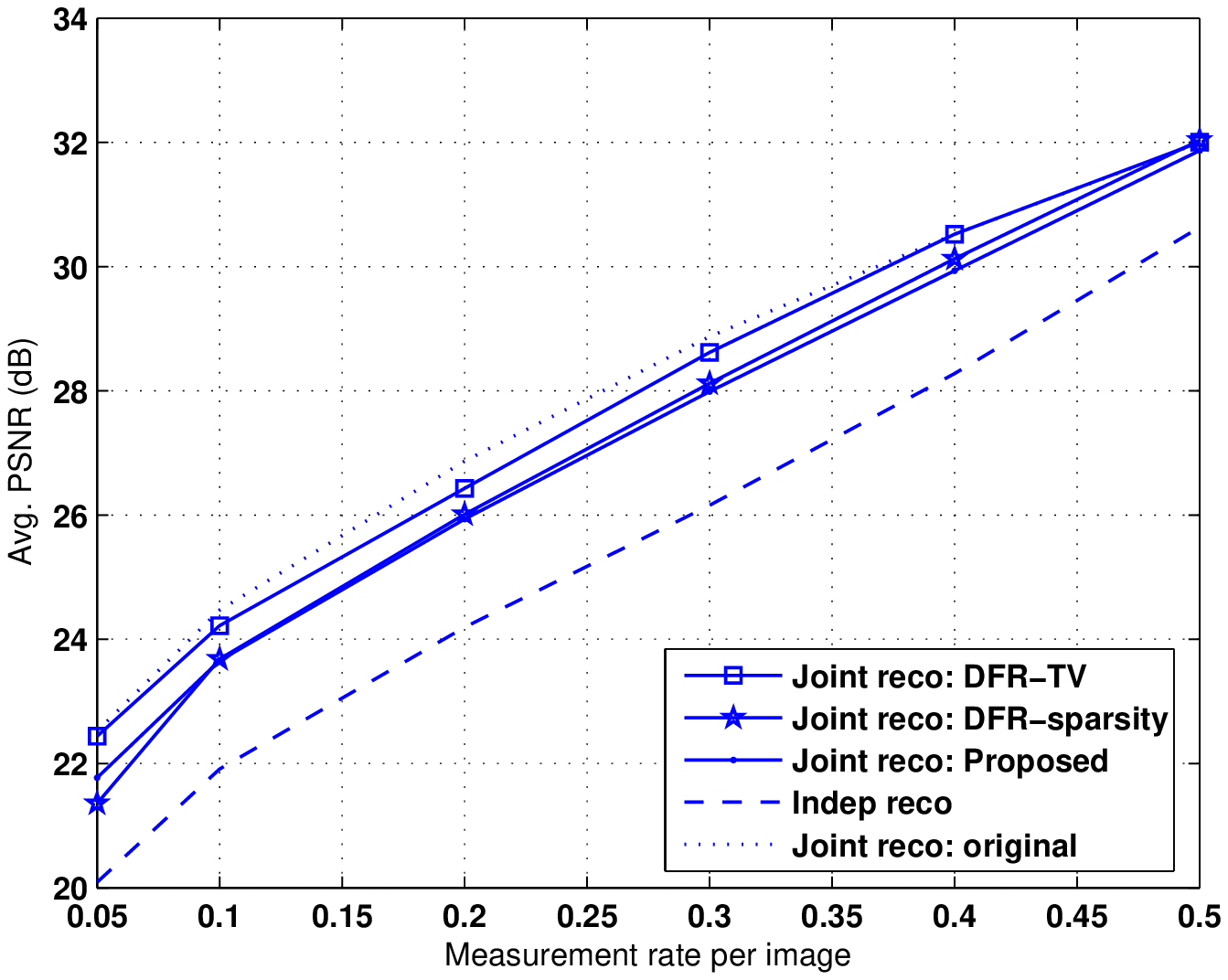} \\
   \mbox{(a)} & \mbox{(b) }
  \end{array}$
 \caption{ Influence of disparity accuracy on the joint reconstruction performance for (a) Tsukuba and (b) Venus datasets. The joint reconstruction is carried out using the dense disparity image in the proposed and DFR-based schemes. The joint reconstruction performance is also compared to the independent reconstruction scheme. }
  \label{Fig:tsuvenu_jr_vs_dfr}
  \end{figure*}

We  analyze the performance of our joint reconstruction scheme with a constraint imposed by the correlation estimated from linear measurements. In particular, we perform joint reconstruction experiments using the correlation model that is estimated using a different measurement matrix for each image, i.e., $\Phi_1 \neq \Phi_2$.  We assume that the image is sparse in an orthonormal basis constructed using a wavelet transform. In our experiments, the parameter $\epsilon$ in the optimization problem is selected based on a trial and error procedure that maximizes the reconstruction image quality $\acute{I_1}$ and $\acute{I_2}$, and we set $\epsilon = 14$. We first compare our results to an independent reconstruction scheme that does not exploit the correlation between the images (i.e., the constraint $\sqnorm{I_2 - A I_1} \leq \epsilon$ is removed in Eq.~(\ref{eqn:jr})). Fig.~\ref{Fig:tsuvenu_jr_vs_dfr} compares the average reconstruction quality between the joint (denoted as \emph{Joint:Proposed}) and independent reconstruction schemes for Tsukuba and Venus datasets. The joint reconstruction improves the reconstruction quality by 2 dB at low measurement rates and about 1 dB at high rates. We also observe in our experiments that the PSNR quality of the reconstructed images $\acute{I_1}$ and $\acute{I_2}$ is similar at a given measurement rate. The disparity estimate thus proves to be useful in improving the quality in the image reconstruction process. 
 
We then jointly reconstruct the images using a disparity estimated with DFR-sparsity and DFR-TV schemes. Fig.~\ref{Fig:tsuvenu_jr_vs_dfr} compares the quality of reconstructed image between the proposed and DFR schemes. We see that the disparity estimated from compressed measurements leads to a competitive performance with the disparity estimated by the DFR-sparsity scheme in terms of joint reconstruction quality. This is particularly obvious at low rate $0.05$ where the DFR-sparsity scheme fails to accurately estimate the disparity. However, the quality of the reconstructed images is marginally penalized (i.e., 0.2 dB and 0.4 dB for the Tsukuba and Venus datasets respectively) compared to the reconstruction achieved when the disparity estimated by the DFR-TV scheme. Finally, we carry out the same experiments in a scenario where the images are jointly reconstructed using a correlation model that is estimated from the original images. This scheme serves as a benchmark for the joint reconstruction since the correlation is accurately known at the decoder. The corresponding results are denoted as  \emph{Joint reco:original} in Fig.~\ref{Fig:tsuvenu_jr_vs_dfr}. We see that the reconstruction quality achieved with the correlation estimated from compressed measurements converges to the performance benchmark when the measurement rate increases which further confirms the quality of the disparity estimation. Finally, it should be noted that similar tendencies have been observed in joint reconstruction of video frames, but the corresponding results are omitted here due to space constraints.

\section{Conclusions} \label{sec:conc}
In this paper we have presented a framework for estimating the correlation between images given in the form of linear measurements without implementing explicit image reconstruction steps. We have proposed a linear representation of disparity and motion models and show that the correlation can be estimated in the compressed domain, thanks to the distance preserving property of the sensing matrix. The correlation is estimated by solving a regularized energy model that enforces consistency with the measurements and smoothness of the correlation information. Extensive experimental results demonstrate that the proposed methodology provides a good estimation of disparity or motion fields in different natural and synthetic image datasets, especially when images are sampled with different measurement matrices. We also show that our correlation estimation solution competes with the correlation estimation solutions that reconstructs images a priori, but becomes clearly advantageous due to its lower computational complexity. The correlation estimation from compressed measurements thus provides an effective solution for distributed scene analysis or coding applications in low complexity sensor networks.

\bibliographystyle{IEEEtran}
\bibliography{mybib}

\begin{thebibliography}{10}
\providecommand{\url}[1]{#1}
\csname url@samestyle\endcsname
\providecommand{\newblock}{\relax}
\providecommand{\bibinfo}[2]{#2}
\providecommand{\BIBentrySTDinterwordspacing}{\spaceskip=0pt\relax}
\providecommand{\BIBentryALTinterwordstretchfactor}{4}
\providecommand{\BIBentryALTinterwordspacing}{\spaceskip=\fontdimen2\font plus
\BIBentryALTinterwordstretchfactor\fontdimen3\font minus
  \fontdimen4\font\relax}
\providecommand{\BIBforeignlanguage}[2]{{%
\expandafter\ifx\csname l@#1\endcsname\relax
\typeout{** WARNING: IEEEtran.bst: No hyphenation pattern has been}%
\typeout{** loaded for the language `#1'. Using the pattern for}%
\typeout{** the default language instead.}%
\else
\language=\csname l@#1\endcsname
\fi
#2}}
\providecommand{\BIBdecl}{\relax}
\BIBdecl

\bibitem{Vijay_ICASSP2011}
V.~Thirumalai and P.~Frossard, ``Dense disparity estimation from linear
  measurements,'' in \emph{Proc. IEEE International Conference on Acoustics,
  Speech and Signal Processing}, 2011.

\bibitem{Donoho}
D.~Donoho, ``Compressed sensing,'' \emph{IEEE Trans. Information Theory},
  vol.~52, pp. 1289--1306, 2006.

\bibitem{Candes}
E.~J. Candes, J.~Romberg, and T.~Tao, ``Robust uncertainty principles: exact
  signal reconstruction from highly incomplete frequency information,''
  \emph{IEEE Trans. Information Theory}, vol.~52, pp. 489--509, 2006.

\bibitem{Candes_imagecoding}
E.~J. Candes and J.~Romberg, ``Practical signal recovery from random
  projections,'' in \emph{Proc. SPIE Computational Imaging}, 2005.

\bibitem{singlepixel_CS}
M.~Duarte, M.~Davenport, D.~Takhar, J.~Laska, T.~Sun, K.~Kelly, and
  R.~Baraniuk, ``Single-pixel imaging via compressive sampling,'' \emph{IEEE
  Signal Processing Magazine}, vol.~25, no.~2, pp. 83--91, 2008.

\bibitem{Mun_blockCS}
S.~Mun and J.~Fowler, ``Block compressed sensing of images using directional
  transforms,'' in \emph{Proc. IEEE International Conference on Image
  Processing}, 2009.

\bibitem{Gan_Eusipco}
L.~Gan, T.~T. Do, and T.~D. Tran, ``Fast compressive imaging using scrambled
  hadamard ensemble,'' in \emph{Proc. European Signal and Image Processing
  Conference}, 2008.

\bibitem{Stankovic}
V.~Stankovic, L.~Stankovic, and S.~Cheng, ``Compressive video sampling,'' in
  \emph{Proc. European Signal and Image Processing Conference}, 2008.

\bibitem{Park}
J.~Y. Park and M.~B. Wakin, ``A multiscale framework for compressive sensing of
  video,'' in \emph{Proc. Picture Coding Symposium}, 2009.

\bibitem{Vaswani}
N.~Vaswani, ``Kalman filtered compressed sensing,'' in \emph{Proc. IEEE
  International Conference on Image Processing}, 2008.

\bibitem{Pudlewski}
S.~Pudlewski, A.~Prasanna, and T.~Melodia, ``Compressed sensing enabled video
  streaming for wireless multimedia sensor networks,'' 2011, to appear in IEEE
  Trans. on Mobile Computing.

\bibitem{Goyal}
V.~K. Goyal, A.~K. Fletcher, and S.~Rangan, ``Compressive sampling and lossy
  compression,'' \emph{IEEE Signal Processing Magazine}, vol.~25, no.~2, pp.
  48--56, 2008.

\bibitem{Duarte_DCS}
M.~F. Duarte, S.~Sarvotham, D.~Baron, M.~B. Wakin, and R.~G. Baraniuk,
  ``Distributed compressed sensing of jointly sparse signals,'' in \emph{Proc.
  Asilomar Conference on Signal System and Computing}, 2005.

\bibitem{Duarte_DCS1}
------, ``Universal distributed sensing via random projections,'' in
  \emph{Proc. Information Processing in Sensor Networks}, 2006.

\bibitem{Kang_ICASSP}
L.~W. Kang and C.~S. Lu, ``Distributed compressive video sensing,'' in
  \emph{Proc. IEEE International Conference on Acoustics, Speech and Signal
  Processing}, 2009.

\bibitem{Do_ICIP}
T.~T. Do, Y.~Chen, D.~T. Nguyen, N.~Nguyen, L.~Gan, and T.~Tran, ``Distributed
  compressed video sensing,'' in \emph{Proc. IEEE International Conference on
  Image Processing}, 2009.

\bibitem{YiMa_PCS}
J.~P. Nebot, Y.~Ma, and T.~Huang, ``Distributed video coding using compressive
  sampling,'' in \emph{Proc. Picture Coding Symposium}, 2009.

\bibitem{Trocan_ICME}
M.~Trocan, T.~Maugey, J.~E. Fowler, and B.~Pesquet-Popescu, ``Disparity
  compensated compressed sensing reconstruction for multiview images,'' in
  \emph{Proc. IEEE International Conference on Multimedia and Expo}, 2010.

\bibitem{Trocan_MMSP}
M.~Trocan, T.~Maugey, E.~W. Tramel, J.~E. Fowler, and B.~Pesquet-Popescu,
  ``Multistage compressed-sensing reconstruction of multiview images,'' in
  \emph{Proc. IEEE International workshop on Multimedia Signal Processing},
  2010.

\bibitem{Vijay_TIP}
\BIBentryALTinterwordspacing
V.~Thirumalai and P.~Frossard, ``Distributed {R}epresentation of
  {G}eometrically {C}orrelated {I}mages with {C}ompressed {L}inear
  {M}easurements,'' \emph{submitted to {IEEE} {T}rans. on {I}mage {P}roc.},
  2011. [Online]. Available: \url{http://arxiv.org/abs/1111.5612}
\BIBentrySTDinterwordspacing

\bibitem{Wakin}
J.~Y. Park and M.~B. Wakin, ``A geometric approach to multi-view compressive
  imaging,'' \emph{to appear in EURASIP Journal on Advances in Signal
  Processing}, 2012.

\bibitem{Li_jr}
X.~Li, Z.~Wei, and Z.~Xiao, ``Compressed sensing joint reconstruction for
  multi-view images,'' \emph{Electronics Letters}, vol.~46, no.~23, pp.
  1548--1550, Nov 2010.

\bibitem{Gan_ICASSP}
T.~Do, T.~Tran, and L.~Gan, ``Fast compressive sampling with structurally
  random matrices,'' in \emph{Proc. IEEE International Conference on Acoustics,
  Speech and Signal Processing}, 2008.

\bibitem{Veksler}
O.~Veksler, ``Efficient graph based energy minimization methods in computer
  vision,'' Ph.D. dissertation, Cornell University, 1999.

\bibitem{Baker_flow}
\BIBentryALTinterwordspacing
S.~Baker, S.~Roth, D.~Scharstein, M.~Black, J.~Lewis, and R.~Szeliski, ``A
  database and evaluation methodology for optical flow,'' \emph{International
  Journal of Computer Vision}, vol.~1, no.~92, pp. 1--31, 2011. [Online].
  Available: \url{http://www.springerlink.com/content/p516733117226378/}
\BIBentrySTDinterwordspacing

\bibitem{Boykov_GC}
Y.~Boykov, O.~Veksler, and R.~Zabih, ``Fast approximate energy minimization via
  graph cuts,'' \emph{IEEE Trans. on Pattern Analysis and machine
  intelligence}, vol.~23, no.~11, pp. 1222--1239, Nov 2001.

\bibitem{Belief_prop}
P.~Felzenszwalb and D.~Huttenlocher, ``Efficient belief propagation for early
  vision,'' \emph{International Journal on Computer Vision}, vol.~70, no.~1,
  pp. 41--54, 2006.

\bibitem{Szeliski_MRF}
R.~Szeliski, R.~Zabih, D.~Scharstein, O.~Veksler, V.~Kolmogorov, A.~Agarwala,
  M.~Tappen, and C.~Rother, ``A comparative study of energy minimization
  methods for markov random fields with smoothness-based priors,'' \emph{IEEE
  Trans. on Pattern Analysis and machine intelligence}, vol.~30, no.~6, pp.
  1068--1080, Jun 2008.

\bibitem{Baraniuk_JLlemma}
R.~Baraniuk, M.~Davenport, and R.~DeVore, ``A simple proof of the restricted
  isometry property for random matrices,'' \emph{Constructive Approximation,
  Springer}, vol.~28, pp. 253--263, Jan 2008.

\bibitem{Gan_JLlemma}
T.~Do, L.~Gan, Y.~Chen, N.~Nguyen, and T.~Tran, ``Fast and efficient
  dimensionality reduction using structurally random matrices,'' in \emph{Proc.
  IEEE International Conference on Acoustics, Speech and Signal Processing},
  2009.

\bibitem{Scharstein}
D.~Scharstein and R.~Szeliski, ``A taxonomy and evaluation of dense stereo,''
  \emph{International Journal on Computer Vision}, vol.~47, pp. 7--42, 2002.

\bibitem{GC_regparam}
L.~Zhang and S.~M. Seitz, ``Estimating optimal parameters for mrf stereo from a
  single image pair,'' \emph{IEEE Trans. on Pattern Analysis and Machine
  Intelligence}, vol.~29, no.~2, pp. 331--342, 2007.

\bibitem{GPSR}
M.~A.~T. Figueiredo, R.~D. Nowak, and S.~J. Wright, ``Gradient projection for
  sparse reconstruction: application to compressed sensing and other inverse
  problems,'' \emph{IEEE Journal of Selected Topics in Signal Processing},
  vol.~1, pp. 586--597, 2007.

\bibitem{bpdqtbx}
\BIBentryALTinterwordspacing
D.~K. Hammond, L.~Jacques, M.~Fadili, G.~Puy, and P.~Vandergheynst, ``The basis
  pursuit dequantizer (bpdq) toolbox,'' July 2009. [Online]. Available:
  \url{http://wiki.epfl.ch/bpdq}
\BIBentrySTDinterwordspacing

\bibitem{nesta}
S.~Becker, J.~Bobin, and E.~Candes, ``{NESTA}: A fast and accurate first-order
  method for sparse recovery,'' Caltech, Tech. Rep., 2009.

\bibitem{Montager}
Y.~L. Montagner, E.~Angelini, and J.-C. Olivo-Marin, ``Comparison of
  reconstruction algorithms in compressed sensing applied to biological
  imaging,'' in \emph{Proc. IEEE International Symposium on Biological
  Imaging}, 2011.

\bibitem{Laurent_BPDNp}
L.~Jacques, D.~K. Hammond, and M.~J. Fadili, ``Dequantizing compressed sensing:
  When oversampling and non-gaussian constraints combine,'' \emph{IEEE Trans.
  on Information Theory}, vol.~57, pp. 559--571, Jan 2011.

\bibitem{Combettes_prox}
\BIBentryALTinterwordspacing
P.~L. Combettes and J.-C. Pesquet, ``Proximal splitting methods in signal
  processing,'' \emph{In Fixed Point Algorithms for Inverse Problems in Science
  and Engineering. Springer,}, 2010. [Online]. Available:
  \url{http://arxiv.org/abs/0912.3522v4}
\BIBentrySTDinterwordspacing

\bibitem{Vijay_ICIP2011}
V.~Thirumalai and P.~Frossard, ``Image reconstruction from compressed linear
  measurements with side information,'' in \emph{accepted to Proc. IEEE
  International Conference on Image Processing}, 2011.

\end{thebibliography}

\end{document}